%% file: main.tex
\documentclass{article}

\usepackage{iclr2026_conference}
\newif\ifpreprint
\preprintfalse
\iclrfinalcopy

\usepackage{times}
\usepackage{inconsolata}
\usepackage[scaled=0.9]{helvet}

\usepackage{graphicx}

\usepackage{amsmath}
\usepackage{amssymb}
\usepackage{natbib}
\usepackage{cases}
\usepackage{hyperref}
\definecolor{accent}{rgb}{0, 0, 0.5}
\hypersetup{colorlinks=true,citecolor=accent,linkcolor=accent,urlcolor=accent}
\usepackage{amsthm} %
\usepackage[capitalize,nameinlink]{cleveref}
\usepackage{mathtools}
\usepackage{mleftright}
\usepackage{stmaryrd} %

\usepackage{doi}

\usepackage[inline]{enumitem}
\newlist{compactenum}{enumerate}{1}
\setlist[compactenum,1]{itemsep=1pt,topsep=0pt,parsep=0pt,partopsep=0pt,label=\arabic*.,leftmargin=2em}
\newlist{compactitem}{itemize}{1}
\setlist[compactitem,1]{itemsep=1pt,topsep=0pt,parsep=0pt,partopsep=0pt,label=$\bullet$,leftmargin=2em}

\usepackage{rycolab-config}
\input{rycolab-macros}

\newcommand{\sympred}[1]{#1}

\makeatletter
\newcommand{\oset}[2]{%
{\mathop{#2}\limits^{\vbox to -.5\ex@{\kern-\tw@\ex@
\hbox{\scriptsize $#1$}\vss}}}}
\makeatother

\newcommand{\countl}{\ensuremath{\oset{\leftharpoonup}{\#}}}

\newcommand{\symprob}[3]{\mymacro{\textnormal{Pr}_{#1}(#3 \mid #2)}}
\newcommand{\sufprob}[3]{\mymacro{\textnormal{Pr}_{#1}(#3 \mid #2)}}
\newcommand{\strprob}[2]{\mymacro{\textnormal{Pr}_{#1}(#2)}}

\newcommand{\OMIT}[1]{}

\newcommand{\revise}[1]{#1}

\makeatletter
\renewenvironment{proof}[1][\proofname]{\par
  \pushQED{\qed}%
  \normalfont \topsep0pt \partopsep0pt
  \trivlist
  \item[\hskip\labelsep\itshape#1\@addpunct{.}]\ignorespaces
}{%
  \popQED\endtrivlist\@endpefalse
}
\makeatother

\title{Probability Distributions Computed \\ by Autoregressive Transformers}

\author{Andy Yang$^1$\quad Anej Svete$^2$\quad Jiaoda Li$^2$\quad Anthony Widjaja Lin$^{3,4}$\quad Jonathan Rawski$^5$ \\ \textbf{Ryan Cotterell}$^2$\quad\textbf{David Chiang}$^1$ \\[1ex] $^1$University of Notre Dame, USA \quad $^2$ETH Zürich, Switzerland \\ $^3$Max-Planck Institute for Software Systems,  Germany \\
$^4$University of Kaiserslautern-Landau, Germany \quad $^5$San José State University,  USA}
\date{July 2025}

\begin{document}

\maketitle

\ifpreprint \lhead{Preprint} \fi

\begin{abstract}
Most expressivity results for transformers treat them as language recognizers---devices that accept or reject strings---rather than as they are used in practice: as language models that generate strings autoregressively and probabilistically. We characterize the probability distributions that transformer language models can express. We show that making transformer language recognizers autoregressive can sometimes increase their expressivity, and that making them probabilistic can break equivalences that hold in the non-probabilistic case. Our overall contribution is to tease apart what functions transformers are capable of expressing in their most common use case as language models.
\end{abstract}

\section{Introduction}

Most work studying transformer expressivity, that is, what classes of computations transformers can perform, treats them as \emph{language recognizers}, where the input is a string and the output is a binary classification: true if the string is accepted and false otherwise \citep{strobl-etal-2024-survey}. 
However, the most common practical use of transformers is as \emph{language models}, which differ in two ways: first, the input is a prefix of a string, and the output is a prediction of the next symbol; second, the prediction is a probability distribution rather than a binary decision.
Such probability distributions, when estimated from large text corpora, have enabled a wide range of applications in natural language processing and beyond.
This paper focuses on a fundamental question: which previous findings on transformer expressivity carry over from the language recognition setting to the language modeling setting? 
On one hand, positive answers validate the utility of previous results on language recognition when studying language models.
On the other, negative answers further underscore the necessity of studying language models \textit{qua} language models.

In order to develop a formal theory of transformers as language models, 
we introduce two distinctions: \emph{unweighted} (or equivalently, \emph{Boolean-weighted}) versus \emph{real-weighted} computation and 
\emph{classifiers}, which map a complete string to a value, versus \emph{autoregressors}, which map each prefix to a distribution over the next token.
This four-way distinction is visualized in \Cref{fig:results}.
Using this terminology, most theoretical work on transformer expressivity \citep[e.g.][]{yang-etal-2024-masked,jerad-etal-2025-unique} focuses on Boolean-weighted classifiers, while practical applications use transformers as real-weighted autoregressors. 
This work investigates whether established expressivity results remain valid when moving from Boolean-weighted to real-weighted transformers, and from classifier settings to autoregressive ones.

We answer these questions for several variants of transformers (see \cref{fig:results}).
\Citet{yang-etal-2024-masked} proved that strictly-masked rightmost unique-hard attention transformers (\uhats), as Boolean classifiers, recognize the same languages as linear temporal logic (\LTL{}) and counter-free automata.
\Citet{jerad-etal-2025-unique} proved that \emph{leftmost} \uhats, as Boolean classifiers, recognize the same languages as a fragment of \LTL{}, called in our notation $\TL[\tlP]$. %
\citet{li-cotterell-2025-characterizing} proved that softmax attention transformers (\smats) with fixed precision, as Boolean classifiers, recognize the same class of languages.
These results carry over easily to real classifiers (\cref{cor:same_models}), with the caveat that there are two commonly-used weighted analogues of 
counter-free automata, deterministic and nondeterministic.
We show that surprisingly, these two diverge in the real-weighted setting,
despite being equivalent in the Boolean setting.
With real weights, \uhats{} are only equivalent to  counter-free DFAs.

This caveat notwithstanding, we may use $\LTL$ and $\mathsf{TL}[\tlP]$ to draw conclusions about the transformer variants listed above. First, real classifiers define some weighted languages that real autoregressors do not, and vice versa (\cref{thm:real_classifier_autoregressor}).  
To pinpoint more precisely how classifiers and autoregressors differ, we turn to Boolean weights. Here, $\uhat$ classifiers and autoregressors are equivalent (\cref{thm:ltl_equiv_altl}). 
But for leftmost \uhats{} and fixed-precision \smats, autoregressors are strictly more expressive than classifiers (\cref{thm:tlp_classifier_autoregressor}).

Similarly, \citet{yang-etal-2025-knee} considered \smats with fixed precision but arbitrary precision inside attention. As Boolean classifiers, such transformers are exactly equivalent to a temporal logic extended with counting operators. But here we show that, as autoregressors, they become slightly more powerful (for a given fixed depth).

Our results largely validate existing results on expressivity of transformers as language recognizers. In many cases, they allow us to transfer results on language recognizers to language models. For example, since \uhats{} as language recognizers cannot recognize PARITY \citep{hahn-2020-theoretical}, neither can \uhats{} as language models. But our results also give good reasons to be cautious about presuming that results on language recognizers directly apply to language models as well.

In \cref{sec:preliminaries}, we define notational preliminaries.
We then (\cref{sec:transformers}) define the classes of transformers we consider and how they can be used as classifiers and as autoregressors. We do this by introducing the notion of a \defn{state encoder},
and in \cref{sec:formalisms} show how two other formalisms, deterministic finite automata (DFAs) and linear temporal logic ($\LTL$), can also be seen as state encoders and therefore can be used as classifiers and autoregressors.
Then (\cref{sec:expressivity}), using \LTL{}, we investigate the expressive power of transformers as both classifiers and autoregressors, yielding the results shown in 
\cref{fig:results}.

\begin{figure}
\centering\small
\tikzset{>=stealth}
\tikzset{edgelabel/.style={auto=left,sloped,font={\tiny},inner sep=0mm,outer sep=0mm}}
\begin{tikzpicture}[x=3.5cm,y=1.75cm]
\tikzset{every node/.style={text width=2.75cm,align=center,anchor=center,execute at begin node={\strut}}}
\node at (0,2.5) {\textit{Boolean classifiers}};
\node(ltlbc) at (0,1) {rightmost \uhats = $\mathsf{LTL}$};
\node(tlpbc) at (0,0) {leftmost \uhats = fixed-precision \smats = $\mathsf{TL[\tlP]}$};
\node[text width=4cm] at (1,2.5) {\textit{Boolean autoregressors}};
\node(ltlba) at (1,1) {rightmost \uhats = $\mathsf{LTL}$};
\node(tlpba) at (1,0) {leftmost \uhats = fixed-precision \smats = $\mathsf{TL[\tlP]}$};
\draw[<->] (ltlbc) edge node[edgelabel]{\cref{thm:ltl_equiv_altl}} (ltlba);
\draw[->] (tlpbc)--(ltlbc);
\draw[->] (tlpba)--(ltlba);
\draw[->] (tlpbc) edge node[edgelabel] {\cref{thm:tlp_classifier_autoregressor}} (tlpba);
\draw[gray!80,xshift=2mm] (1.5,-0.25) -- (1.5,2.5);
\node at (2,2.5) {\textit{Real classifiers}};
\node(ltlrc) at (2,1) {rightmost \uhats = 
\LTL{} = \mbox{cfDFAs} = cfNFAs};
\node(tlprc) at (2,0) {leftmost \uhats = fixed-precision \smats = $\mathsf{TL[\tlP]}$};
\node[text width=4cm] at (3,2.5) {\textit{Real autoregressors}};
\node(nfa) at (3,2) {cfNFAs};
\node(ltlra) at (3,1) {rightmost \uhats = \LTL = cfDFAs};
\node(tlpra) at (3,0) {leftmost \uhats = fixed-precision \smats = $\mathsf{TL[\tlP]}$};
\draw[draw=none] (tlprc) to node[midway] {\scalebox{2}[0.75]{$\neq$}} node[edgelabel] {\cref{thm:real_classifier_autoregressor}} (tlpra);
\draw[draw=none] (ltlrc) to node[midway] {\scalebox{2}[0.75]{$\neq$}} node[edgelabel] {\cref{thm:real_classifier_autoregressor}} (ltlra);
\draw[->] (tlprc)--(ltlrc);
\draw[->] (tlpra)--(ltlra);
\draw[->] (ltlra) edge node[edgelabel] {\cref{thm:cfdfa_to_cfnfa}} (nfa);
\end{tikzpicture}
\DeclareRobustCommand{\rarr}{\tikz[baseline=-0.8mm]{\draw[->] (0,0)--(3mm,0);}}
\DeclareRobustCommand{\harr}{\tikz[baseline=-0.8mm]{\draw[<->] (0,0)--(4mm,0);}}
\caption{In the Boolean semiring, equivalences from the literature \citep{yang-etal-2024-masked,jerad-etal-2025-unique,yang-etal-2025-knee} carry over from classifiers to autoregressors; however, sometimes autoregressors are more expressive than classifiers. 
In the real semiring, \LTL{} and counter-free DFAs and NFAs become less expressive than counter-free NFAs, and rightmost \uhats are only as expressive as the former. 
Key: \rarr{} strict inclusion, \harr{} equivalence.
$\neq$ incomparable.} \label{fig:results}
\end{figure}

\section{Preliminaries}
\label{sec:preliminaries}

Throughout this paper, we work with weighted languages. We define some key concepts here, but for a more detailed introduction, see the handbook chapter by \citet{droste-kuich-2009}.

Let $\alphabet$ be an \defn{alphabet}, that is, a finite, non-empty set of \defn{symbols}, and let $\kleene{\alphabet}$ 
be the set of strings over $\alphabet$.
We often augment $\alphabet$ with start and end symbols $\bos$ and $\eos$, but never consider $\bos$ or $\eos$ to belong to $\alphabet$. 
For any string $\str=w_1\cdots w_n$, we write the length of $\str$ as $|\str|=n$. We write $\str_{<i} = w_1 \cdots w_{i-1}$ and $\str_{\le i} = w_1 \cdots w_i$. We write $\epsilon$ for the empty string.

We think of weights and probabilities as elements of \defn{semirings}, an abstraction of the usual addition and multiplication operations that allows results and algorithms to apply generically to multiple settings.
A semiring $\semiring$ has an addition operation $\oplus$, additive identity $\zero$, multiplication operation $\otimes$, and multiplicative identity $\one$. The two semirings we focus on in this paper are the \defn{(extended nonnegative) real semiring} $\realsemiring$, which contains all nonnegative real numbers and $+\infty$, and in which $\oplus$ and $\otimes$ are real addition and multiplication; and the \defn{Boolean semiring} $\B$, in which $\oplus$ is disjunction ($\lor$), $\zero$ is false ($\false$), $\otimes$ is conjunction ($\land$), and $\one$ is true ($\true$).

A \defn{weighted language} (also called a \defn{formal power series}) is a function $\fpseries \colon \kleene{\alphabet} \rightarrow \semiring$.
When $\semiring$ is complete, that is, when $\semiring$ is closed under infinite summations, as $\realsemiring$ and $\B$ are, we call a weighted language \defn{normalized} if
$\sum_{\str \in \kleene{\alphabet}} \fpseries(\str) = \one$.

For sets $X$ and $Y$, we write $Y^X$ for the set of functions from $X$ to $Y$, and $2^X$ for %
the power set of~$X$. For any proposition $\phi$, we write $\ind{\phi}$ to be $1$ if $\phi$ is true and $0$ if $\phi$ is false.

\section{Transformer Language Models}
\label{sec:transformers}

In this section, we recall the definition of transformers that we will use throughout most of this paper. We also distinguish between two ways that transformers (and other formalisms) can be used to define weighted languages.

\subsection{Unique Hard Attention Transformers}

Following \citet{yang-etal-2024-masked}, we use \defn{unique-hard attention transformers} (\uhats), specifically, with rightmost-hard attention, strict future masking, and no position embeddings. We give a definition of strictly masked rightmost-hard attention here; for a definition of the rest of the network, see, for example, the survey by \citet{strobl-etal-2024-survey}.

\newcommand{\dkey}{{d_{\textrm{k}}}}

The attention function receives a sequence of query vectors $\mathbf{q}^{(i)} \in \R^\dkey$, key vectors $\mathbf{k}^{(j)} \in \R^\dkey$, and value vectors $\mathbf{v}^{(j)} \in \R^d$, all of which are column vectors, for $i, j \in [n]$.
At each position $i$, it computes a sequence of vectors
\begin{align}
\textnormal{Att}\left((\mathbf{q}^{(i)})_{i \in [n]}, (\mathbf{k}^{(j)})_{j \in [n]}, (\mathbf{v}^{(j)})_{j \in [n]}\right) &= (\mathbf{c}^{(i)})_{i \in [n]}
\end{align}
where
\begin{align*}
a_i(j) &= \mathbf{q}^{(i)} \cdot \mathbf{k}^{(j)} && \text{is an attention score for each position $j$,}\\
a^*_i &= \max_{j < i} a_i(j) && \text{is the maximum attention score,} \\
j_i &= \max \{ j < i \mid a_i(j) = a^*_i \} && \text{is the rightmost maximum-scoring position, and} \\
\mathbf{c}^{(i)} &= \begin{cases}
\mathbf{v}^{(j_i)} & \text{if $i>0$} \\
\mathbf{0} & \text{if $i=0$}
\end{cases} && \text{is the attention output.}
\end{align*}

\revise{A transformer $\mathcal{T}$ consists of a word embedding $\textnormal{Emb}\colon \alphabet\to \R^d$, followed by a composition of attention functions and feed-forward networks, allowing Layernorm but no using positional encodings.
We defer the definitions because they are standard, and we primarily will refer to the expressive equivalence of these transformers and different formal logics as shown by \citet{yang-etal-2024-masked,jerad-etal-2025-unique,li-cotterell-2025-characterizing}.
Given an input string  $\str=w_1\cdots w_n$, a transformer $\mathcal{T}$ 
prepends a symbol $w_0 = \bos$ and computes a sequence of states $\mathcal{T}(\str)= (\statevec^{(0)},\ldots,\statevec^{(n)})$, where $\statevec^{(i)} \in \R^d$ is the state after reading $w_i$.}
There are at least two ways to use $\mathcal{T}$ to define a weighted language, which we describe below.\footnote{A third intermediate way would be to multiply the weights at each position like an autoregressive model, but not to pass the output symbol at each position autoregressively to the input at the next position. %
Although interesting in its own right, it has not, to our knowledge, been used with any neural sequence models, and we do not explore this style of model here.}

\subsection{Classifiers}

The first way that a transformer can define a weighted language is as a \defn{classifier}. 

\begin{definition} \label{def:uhat_classifier}
A \uhat{} \defn{classifier} is a pair $C = (\mathcal{T}, c)$, where $\mathcal{T} \colon \Sigma^* \to (\R^d)^*$ is a $\uhat$ and $c \colon \R^d \to \semiring$ outputs a scalar weight at the last position only:
\begin{equation}
    C(\str) = c(\mathcal{T}(\str)_n).
\end{equation}
\end{definition}
For the Boolean semiring $(\mathbb{K}=\B)$, we accept a string iff the transformer outputs $\true$ at the last position. For example, the output function could be $c(\mathbf{y}) = \ind{\mathbf{w}\cdot\mathbf{y} + b \ge 0}$, where $\mathbf{w} \in \R^d$ and $b \in \R$ are parameters. This is the setup used for binary classification with a transformer encoder \citep{devlin2019bert} and in most theoretical papers on transformer expressivity.

\subsection{Autoregressive Models}

The second way for a transformer to define a weighted language is as an \defn{autoregressive model}, or an \defn{autoregressor} for short (by analogy with \emph{classifier}). An autoregressor pairs a $\uhat$ encoder with an output function $a\colon\mathbb{R}^d\to \semiring^{\alphabet\cup\{\eos\}}$, which outputs at each position a weight distribution for the next symbol, including $\eos$. In the real semiring ($\mathbb{K} = \realsemiring$), a typical example of such an output function is $\arout(\statevec) = \softmax (\mathbf{W}\statevec + \mathbf{b})$.

To line up with the more familiar notation of conditional probability distributions, we write, for all $\sym\in\alphabet\cup\{\eos\}$,
\begin{align} 
\symprob{\autoregressor}{\str_{\le i}}{\sym} &= \arout(\mathcal{T}(\str)_i)(\sym). \label{eq:symprob}
\end{align}
This is well-defined because $\mathcal{T}(\str)_i$ depends only on $\str_{\le i}$, that is, $\str_{\le i} = \str'_{\le i} \iff \mathcal{T}(\str)_i = \mathcal{T}(\str')_i$.
As suggested by this notation, we want 
$\symprob{\autoregressor}{\stru}{\mathord\cdot}$ to be a probability distribution over $\alphabet\cup\{\eos\}$. But we impose a stronger condition.  First, we extend $\symprob{\autoregressor}{\stru}{\sym}$ to the probability distribution of suffixes given (possibly empty) prefixes:
\begin{align}
\sufprob{\autoregressor}{\stru}{\strv}&=\left(\bigotimes_{i=1}^{|\strv|}\symprob{\autoregressor}{\stru \strv_{<i}}{v_{i}}\right)\otimes\symprob{\autoregressor}{\stru\strv}{\eos} \label{eq:sufprob} \\
\strprob{\autoregressor}{\str}&=\sufprob{\autoregressor}{\epsilon}{\str}.
\end{align}
Then we require that every such distribution sums to one:
\begin{definition} \label{def:uhat_autoregressor}
A \uhat{} \defn{autoregressor} over a complete semiring $\semiring$ is a pair $\autoregressor = (\mathcal{T}, \arout)$, where $\mathcal{T} \colon \Sigma^* \to (\R^d)^*$ is a $\uhat$, and $\arout \colon \R^d \to \semiring^{\alphabet\cup\{\eos\}}$ is a function such that for all $\stru \in \alphabet^*$, using the notation of \cref{eq:symprob,eq:sufprob},
\begin{equation}
\bigoplus_{\strv \in \alphabet^*} \sufprob{\autoregressor}{\stru}{\strv} = \one. \label{eq:suffix_probability}
\end{equation}
\end{definition}

This implies that:
\begin{itemize}
\item An autoregressor generates strings symbol by symbol. That is, for all prefixes $\stru$, %
\begin{equation}
\bigoplus_{\sym \in \alphabet\cup\{\eos\}} \symprob{\autoregressor}{\stru}{\sym} = \one. \label{eq:sym_probability}
\end{equation}
\item An autoregressor does not have any dead ends or endless loops. That is, for all prefixes $\stru$,
\begin{align}
\bigotimes_{i=1}^n \symprob{\autoregressor}{\stru_{<i}}{u_i} \ne \zero &\implies  \strprob{\autoregressor}{\stru\strv} \ne \zero~\text{for some suffix $\strv$.}
\end{align}
\item An autoregressor defines a normalized weighted language.
\end{itemize}

\section{Other Formalisms}
\label{sec:formalisms}

We can analogously use other formalisms to define classifier or autoregressive models. We generalize from transformers to other formalisms by means of the following notion.
\begin{definition}
A \defn{\strtostate} is a function that sends a string $w_1 \cdots w_n$ to a sequence of states $\statey_0, \ldots, \statey_n \in Q$ (where $Q$ is a finite or infinite set of states) such that $q_i$ depends only on $\str_{\le i}$. 
\end{definition}
Like transformers, any \strtostate{} can be equipped with an output function $c \colon Q \to \semiring$ to give a classifier model (as in \cref{def:uhat_classifier}) or $\arout \colon Q \to \semiring^{\alphabet\cup\{\eos\}}$ (as in \cref{def:uhat_autoregressor}) to give an autoregressor.\looseness=-1

\subsection{Finite Automata}
\label{sec:automata}

We give brief definitions of weighted and counter-free deterministic finite automata. For a fuller treatment, please see the handbook chapter by \citet{DK21} and the monograph by \citet{mcnaughtonpapert1971}.

\begin{definition}
A \defn{deterministic finite automaton} (DFA) is a tuple $\automaton=(\alphabet, Q, \trans, \initstate)$, where
    \begin{compactitem}
        \item $\alphabet$ is an alphabet,
        \item $\states$ is a finite set of \defn{states},
        \item $\trans \colon \states \times \alphabet \to \states$ is a \defn{transition function}, and 
        \item $\initstate \in \states$ is the \defn{initial} state.
    \end{compactitem}
    We extend $\trans$ to a mapping $\trans^*\colon Q\times \alphabet^*\to Q$ such that:
    \begin{equation} \begin{aligned}
        \trans^*(q,\epsilon)&=q\\
        \trans^*(q, \sym \str)&= \trans^*(\trans(q,\sym), \str).
    \end{aligned} \end{equation}
\end{definition}
A DFA $\automaton$ defines a \strtostate{}
\begin{align}
\automaton\colon \alphabet^*&\to Q^* \notag \\
    \automaton(\str)_i & = \begin{cases}
        \initstate & i=0\\
        \trans^*(\initstate, w_1 \cdots w_i)&  0<i\leq n.
    \end{cases}
\end{align}
A DFA with classifier outputs in the Boolean semiring is the same as the standard definition of a DFA: the states that output $\true$ are the accept states, and the states that output $\false$ are the reject states.
A DFA with autoregressive outputs in the real semiring is the same as the standard definition of a weighted DFA: when it is in state $\stateq$, the next input symbol $\sym$ determines both the next state $\trans(\stateq, \sym)$ as well as the symbol weight $\arout(\stateq)(\sym)$. Moreover, each state has an accepting weight $\arout(\stateq)(\eos)$.

In this paper, we are only interested in the following subclass of finite automata called \defn{counter-free automata}, which we abbreviate as cfDFAs.
\begin{definition}
     We say that a DFA with transition function $\trans$ is \defn{counter-free} if there exists some $k$ such that for all states $q$ and all strings~$\str$, we have $\trans^*(q,\str^{k})=\trans^*(q,\str^{k+1})$. 
\end{definition}
Examples of counter-free and non-counter-free DFAs are shown in \cref{fig:automata-examples}ab.

\begin{figure}\centering
\tikzset{state/.append style={minimum size=8mm,inner sep=1mm}}
\begin{tabular}{ccc}
\begin{tikzpicture}[x=2cm,baseline=0cm]
\node (q1) [initial,state] at (0,0) {$q_1$};
\node (q2) [state] at (1,0) {$q_2$};
\draw[transition] (q1) edge[bend left] node[auto] {$a$} (q2);
\draw[transition] (q2) edge[bend left] node[auto] {$b$} (q1);
\end{tikzpicture} &
\begin{tikzpicture}[x=2cm,baseline=0cm]
\node (q1) [initial,state] at (0,0) {$q_1$};
\node (q2) [state] at (1,0) {$q_2$};
\draw[transition] (q1) edge[bend left] node[auto] {$a$} (q2);
\draw[transition] (q2) edge[bend left] node[auto] {$a$} (q1);
\end{tikzpicture}
&
    \begin{tikzpicture}[baseline=0cm]
        \node[state, initial] (q0) {$\stateq_0/0$};
        \node[state, accepting] (q1) [above right = 0.25cm and 1.5cm of q0] {$\stateq_1/\frac{1}{2}$};
        \node[state, accepting] (q2) [below right = 0.25cm and 1.5cm of q0] {$\stateq_2/\frac{1}{4}$};
        \draw[transition] (q0) edge[auto] node{$a/\frac{1}{2}$} (q1)
        (q0) edge[auto=right] node{$a/\frac{1}{2}$} (q2)
        (q1) edge[auto, loop right] node{$\syma/\frac{1}{2}$} (q1)
        (q2) edge[auto, loop right] node{$\syma/\frac{3}{4}$} (q2);
    \end{tikzpicture} \\
(a) & (b) & (c)    
\end{tabular}
\caption{(a) A DFA that is counter-free (with $k=2$). (b) A DFA that is not counter-free, because for all $k$, the strings $a^k$ and $a^{k+1}$ have opposite actions. (c) A counter-free weighted NFA that has no equivalent weighted DFA (\cref{thm:cfdfa_to_cfnfa}).}
\label{fig:automata-examples}
\end{figure}

Counter-free \defn{nondeterministic finite automata} (NFAs), in which a state can have more than one outgoing transition with the same symbol are equivalent to counter-free DFAs \citep{mcnaughtonpapert1971}.
For a definition and proof of equivalence, please see \cref{sec:nfa_proof}.

\subsection{Linear Temporal Logic}

We give a brief definition of linear temporal logic and its fragments. For a fuller treatment, please see the article by \citet{sep-temporal-logic}.

\begin{definition}
The formulas of past $\LTL$ are defined by the grammar
\begin{align*}
    \phi &\bnfto  \lnot \phi_1 \mid \phi_1\land \phi_2 \\
    &\bnfor\sympred{\sym} &&\sym \in \alphabet \\
    &\bnfor\sympred{\bos} && \text{Beginning of string} \\
    &\bnfor \tlY \phi_1 && \text{Yesterday} \\
    &\bnfor \tlH \phi_1 && \text{Historically} \\
    &\bnfor \phi_1 \tlS \phi_2 && \text{Since}
\end{align*}
Formulas $\true$ (true), $\false$ (false), $\phi_1 \lor \phi_2$, $\phi_1 \leftrightarrow \phi_2$, and so on, can be defined as syntactic sugar in terms of the above. The temporal operator $\tlP \phi$ (which holds iff $\phi$ was Previously true at some time) can be defined as $\tlH \phi = \neg (\tlP (\neg \phi))$.

The semantics of formulas is given by the relation $\str, i \models \phi$ (``$\str$ satisfies $\phi$ at position $i$''), defined as follows:
\begin{subequations}
\begin{alignat}{2}
    &\str, i \models \lnot \phi_1 &\iff& \str, i \not\models \phi_1 \label{eq:models_not}\\
    &\str, i \models \phi_1 \land \phi_2 &\iff& \text{$\str, i \models \phi_1$ and $\str, i \models \phi_2$} \label{eq:models_and} \\
    &\str, i \models \sympred{\bos} &\iff& i = 0 \label{eq:models_bos} \\
    &\str, i \models \sympred{\sym} &\iff& w_i = \sym \label{eq:models_sym} \\
    &\str, i \models \tlY \phi_1 &\iff& \text{$i>0$ and $\str, i-1 \models \phi_1$} \label{eq:models_Y}\\
    &\str, i \models \tlH \phi_1 &\iff& \text{$\str, j \models \phi_1$ for all $j\le i$} \label{eq:models_H}\\
    &\str, i \models \phi_1 \tlS \phi_2 &\iff& \text{($\str,j \models \phi_2$ for some $j \le i$) and ($\str,j' \models \phi_1$ for all $j<j'\le i$).} \label{eq:models_S}
\end{alignat}
\end{subequations}
We write $\str \models \phi$ as shorthand for $\str, |\str| \models \phi$. 
\end{definition}

For any set of operators $\mathcal{O} \subseteq \{\mathord{\tlY}, 
\mathord{\tlH}, \mathord{\tlS}\}$, we write $\TL[\mathcal{O}]$ for the set of formulas using only operators in $\mathcal{O}$. Thus $\text{past $\LTL$} = \TL[\mathord{\tlY},\mathord{\tlS}]$.
Given a tuple of formulas $\Phi=(\phi_1, \ldots, \phi_m)$, we can define a \strtostate{} (typically we will use $\Phi$ to refer to both the state encoder and tuple of formulas to make their connection more obvious)
\begin{align}
    \Phi\colon \alphabet^*&\to (\B^{m})^* \notag \\
    \Phi(\str)_i &= (\ind{\str, i \models \phi_1}, \ldots \ind{\str, i \models \phi_m}). \label{eq:ltl_state_encoder}
\end{align}
We typically write $(\Phi,\arout)$ for the autoregressor using the state encoder induced by $\Phi$.

\Citet{droste-gastin-2019-aperiodic} define a weighted first-order logic, with several variations corresponding to several subclasses of weighted counter-free automata. \Citet{mandrali2013characterizations,mandrali2015weighted} do the same for \LTL{}. 
Both of these logics have, roughly speaking, four layers: (1) a core Boolean logic, (2) weights conditioned on formulas, (3) products over positions, and (4) addition and sums over positions.
This is similar to our framework, which has (1) a core Boolean logic, (2) classifier output functions that can choose weights conditioned on formulas, and (3) autoregressive output functions that can also compute products over positions.

\section{Expressivity Results}
\label{sec:expressivity}

Previous results have shown that \uhats, $\LTL$, and cfDFAs are equivalent in terms of language recognition. 
In \cref{sec:strtostates_expressivity}, we use the results to show that these formalisms are also equivalent as weighted classifiers and as autoregressors.

Next, we compare the expressivity of classifier versus autoregressive models.
Given the equivalence of the above formalisms, we will mainly discuss $\LTL$.
In the real semiring (\cref{sec:real}), \LTL{} classifiers define exactly the aperiodic step functions (defined below), which are less expressive than \LTL{} autoregressors. And \LTL{} autoregressors, in turn, are equivalent to counter-free DFA autoregressors and less expressive than weighted counter-free NFAs.

In the Boolean semiring, \LTL{} classifiers and autoregressors are equivalent, which is the main result of \cref{sec:autoregressor_expressivity}. 
However, when we consider fragments of \LTL, this equivalence breaks down, and autoregressors may become more expressive than classifiers (\cref{sec:ltl_fragments}). Similarly, in the temporal logic with counting $\TL[\countl]$ and the programming language C-RASP \citep{yang2024counting}, autoregressors are more expressive than classifiers (\cref{sec:crasp}).

\subsection{\Strtostatestitle{}}
\label{sec:strtostates_expressivity}

We say that two \strtostates{} $\tau_1\colon\alphabet^*\to Q_1^*$ and $\tau_{2}\colon\alphabet^*\to Q_2^*$ are \defn{equivalent} if there is a bijection $f\colon Q_1\to Q_2$ such that for all $\str\in \alphabet^*$, $f(\tau_1(\str))=\tau_2(\str)$.

\begin{restatable}{theorem}{stateEquivalence}\label{thm:same_states}
\uhats, $\LTL$, and cfDFAs define equivalent \strtostates{}. 
\end{restatable}
\begin{proof}
See \cref{sec:same_states_proof}. The proof is an adaptation of existing results \citep{yang-etal-2024-masked,schutzenberger:1965,mcnaughtonpapert1971,kamp:1968} connecting \uhats, $\LTL$ and cfDFAs as language recognizers.
\end{proof}

The following is an immediate consequence of \cref{thm:same_states} and the definitions of classifier and autoregressive models.

\begin{corollary}\label{cor:same_models}
\uhats, \LTL{}, and cfDFAs as classifier models define the same weighted languages. Similarly when they are used as autoregressive models. 
\end{corollary}
\begin{proof}
    By the previous theorem, all these formalisms define equivalent \strtostates{}. 
    Therefore there exist output functions with which they define the same weighted languages.
\end{proof}

\subsection{Real classifiers and autoregressors} \label{sec:real}

In this section, we consider weights in the real semiring. We characterize what weighted languages can be expressed, first by real classifiers, then by real autoregressors.

\begin{definition}
An \defn{aperiodic step function} \citep{droste-gastin-2008} is a weighted language $S \colon \alphabet^* \to \mathbb{K}$ such that $S(\str) = \bigoplus_{i=1}^m \weight_i \otimes \ind{\str \in L_i}$ where $k_1, \ldots, k_m \in \mathbb{K}$ are constants and $L_1, \ldots, L_m$ are aperiodic, that is, counter-free, regular languages.
\end{definition}

\begin{proposition} \label{thm:classifier}
An \LTL{} classifier defines the aperiodic step functions.
\end{proposition}

\begin{proof}
Given any aperiodic step function as defined above, we can write, for each $L_i$, an \LTL{} formula $\phi_i$. Then we can write a classifier output function $c(\statevec) = \bigoplus_{i=1}^m \weight_i \otimes \statescl_i$. 

Conversely, given an \LTL{} classifier consisting of a tuple of formulas $(\phi_1, \ldots, \phi_m)$ and an output function $c(\statevec)$, 
for every $\statevec \in \B^{[m]}$, write the formula $\phi_\statevec = \bigwedge_{i=1}^m (\phi_i \leftrightarrow \statescl_i)$. For every $\statevec$, let $L_\statevec$ be the language defined by $\phi_\statevec$.
Then the weighted language can be written as the step function
$S(\str) = \bigoplus_{\statevec \in \B^{[m]}} c(\statevec) \otimes \ind{\str \in L_\statevec}.$
\end{proof}

The following easy corollary of \cref{thm:classifier} shows that autoregressors and classifiers are incomparable. It makes use of weighted regular expressions \citep{Sakarovitch2009}, in which the expression $\sigma$ (for any $\sigma \in \Sigma$) matches symbol $\sigma$ with weight $\one$, while the expression $k$ (for any $k \in \mathbb{K}$) matches $\epsilon$ with weight $k$.
\begin{corollary} \label{thm:real_classifier_autoregressor}
In the real semiring: (a) The weighted language $(\frac12 a)^*$ is expressible by an $\LTL$ autoregressor, but not by any $\LTL$ classifier. 
(b) The language $(1a)^*$ is expressible by an $\LTL$ classifier but not any $\LTL$ autoregressor.

Both (a) and (b) hold with $\LTL$ replaced by $\mathsf{TL}[\tlH]$.
\end{corollary}
\begin{proof}
The first language has an infinite number of string weights, but an aperiodic step function can only output a finite number of different weights. 
On the other hand, it is easy to write an  $\LTL$ (or $\TL[\tlH]$) autoregressor to recognize this.
The second language can easily be expressed by a classifier assigning weight $1$ to every string of zero or more $a$'s, but is not expressible by any autoregressor because it is not a normalized weighted language.
\end{proof}

As real autoregressors, $\LTL$ formulas are equivalent to counter-free DFAs by \cref{cor:same_models}. However, there are several nonequivalent weighted analogues of counter-free automata \citep{droste-gastin-2008}, and $\LTL$ and $\uhat$ autoregressors are only equivalent to the least powerful of these. In particular, both are less expressive than weighted counter-free NFAs.

\begin{proposition} \label{thm:cfdfa_to_cfnfa}
Weighted counter-free NFAs define more weighted languages than counter-free DFA autoregressors do.
\end{proposition}
\begin{proof}
See \cref{sec:nfa_proof}.
\Cref{fig:automata-examples}c shows an example of a counter-free weighted NFA that is not determinizable.
\end{proof}

\subsection{Boolean classifiers and autoregressors} \label{sec:boolean}

To examine more carefully how autoregressors add expressivity, we turn to the Boolean semiring.
We will see that $\LTL$ classifiers and $\LTL$ autoregressors are equivalent, but with an important caveat: with certain fragments and extension of $\LTL$ that use only a subset of the temporal operators, autoregressors can be more expressive than classifiers. These variants of $\LTL$ are particularly interesting because they have been proven to be equivalent to variants of transformers.

\subsubsection{$\LTL$}
\label{sec:autoregressor_expressivity}

\newcommand{\nextsym}[1]{\mathrm{next}_{#1}}
\newcommand{\pref}{\mathrm{prefix}}

In the Boolean semiring, $\LTL$ classifiers and autoregressors are equivalent, but the conversion from an autoregressor to a classifier uses the $\tlY$ and $\tlH$ operators.

\begin{restatable}{theorem}{ltlEquivaltl}\label{thm:ltl_equiv_altl}
    For any set of operators $\mathcal{O} \subseteq \{\mathord{\tlY}, 
        \mathord{\tlH}, \mathord{\tlS}\}$:
        
        \begin{compactenum}[label=(\alph*)]
            \item\label{thm:ltl_to_altl} 
            For any nonempty language $L$ defined by a Boolean-weighted $\TL[\mathcal{O}]$ classifier, there exists a Boolean-weighted $\TL[\mathcal{O}]$ autoregressor defining the same language $L$.
            \item\label{thm:altl_to_ltl} For any language $L$ defined by a Boolean-weighted $\TL[\mathcal{O}]$ autoregressor,  there exists a Boolean-weighted $\TL[\mathcal{O} \cup \{\tlY,\tlH\}]$ classifier defining the same language $L$.
        \end{compactenum}
        
\end{restatable}
\begin{proof}
See \cref{sec:ltl_equiv_altl_proof} for the full proof; a proof sketch follows.

To prove \labelcref{thm:ltl_to_altl}, 
we need to construct an autoregressor that tests, given any position $i$ and symbol $\sym$, whether $\str_{\le i}\sym$ is a prefix of some string that is accepted by the classifier.
To do this, we introduce two new operators as syntactic sugar that do not increase the expressivity of the logic:
\begin{align*}
\str \models \nextsym{\sym}(\phi) &\iff \str \sym \models \phi \\
\stru \models \pref(\phi) &\iff \text{there exists $\strv \in \kleene{\alphabet}$ such that $\stru\strv \models \phi$}.
\end{align*}
The $\nextsym{\sym}$ operator is what lets us hypothesize $\sym$ as the next symbol, and the $\pref$ operator is what lets us hypothesize the rest of the string.

To prove \labelcref{thm:altl_to_ltl}, we need to construct a classifier that tests whether, for every position $i$, the autoregressor predicts that $\str_{< i}$ can be followed by $\str_{i}$. To test the relationship between each prefix and the next symbol, we use the $\tlY$ operator, and to do so at every position, we use the $\tlH$ operator.
\end{proof}

From \cref{thm:ltl_equiv_altl}, we can conclude that for $\uhats$, which are equivalent to $\LTL$, autoregression does not add any expressivity. On other transformer variants, please see \cref{sec:ltl_fragments}.

The construction that desugars $\pref(\phi)$ into a formula of $\mathsf{TL}[\mathcal{O}]$
yields a formula whose size is exponential in that of $\phi$. To shed light on whether this bound is tight, we show the following.
\begin{restatable}{proposition}{tails}\label{cor:ltl_tails}
\begin{compactenum}[label=(\alph*)]
\item\label{subthm:ltl_tails_HY} There does not exist a transformation $\pref'$ such that $\pref'(\phi)$ is constructible in polynomial time (in $|\phi|$) and satisfies \cref{eq:def_pref} for every formula $\phi$ in $\TL[\tlH,\tlY]$, unless $\PTIME=\PSPACE$.
\item\label{subthm:ltl_tails_H} Similarly for $\TL[\tlH]$, unless $\PTIME=\NPTIME$.
\item\label{subthm:ltl_tails_Y} Similarly for $\TL[\tlY]$, unless $\PTIME=\NPTIME$.
\end{compactenum}
\end{restatable}
\begin{proof}
See \cref{sec:ltl_tails_proof}. This is a reduction from existing results on the hardness of testing whether a formula defines an empty language \citep{DeGiacomoVardi2013,FiondaGreco2016}.
\end{proof}

Note that we have only shown (conditionally) that constructing $\pref'(\phi)$ requires super-polynomial time; it's possible that $\pref'(\phi)$ is short but difficult to construct.

\subsubsection{Fragments of \LTL}
\label{sec:ltl_fragments}

\Cref{thm:ltl_equiv_altl} shows that \LTL{} classifiers and autoregressors are equivalent, and this remains true for some fragments of \LTL. But the asymmetric conditions of the theorem suggest that when the set of operators $\mathcal{O}$ lacks either $\tlH$ or $\tlY$, Boolean autoregressors are more expressive than classifiers. In this section, we prove that this is indeed the case.

Moreover, such fragments are relevant to the study of transformers.
\Citet{li-cotterell-2025-characterizing} show that fixed-precision future-masked transformers are equivalent to $\TL[\tlP]$, which is in turn equivalent to $\TL[\tlH]$. Similarly, \citet{jerad-etal-2025-unique} show that future-masked leftmost-hard attention transformers are also equivalent to $\TL[\tlP]$.

\begin{proposition} \label{thm:tlp_classifier_autoregressor}
    The language $(ab)^*$ is defined by a Boolean $\TL[\emptyset]$ autoregressor but not defined by any $\TL[\tlH]$ or $\TL[\tlY]$ classifier.
\end{proposition}
\begin{proof} 
    Consider the \strtostate{} induced by the triple of formulas $\Phi=(\sympred{\bos},\sympred{a},\sympred{b})$ as in \cref{eq:ltl_state_encoder},
    \begin{align*}
    \Phi(\str)_i &= (\ind{\str, i \models \sympred{\bos}}, \ind{\str, i \models \sympred{a}}, \ind{\str, i \models \sympred{b}})
    \end{align*}
    \begin{align*}
    \arout \colon \B^3 &\to \B^{\{a,b,\eos\}} \\
        \arout((q_\bos,q_a,q_b))(a)&=\top\iff \text{$q_\bos=\top$ or $q_b=\top$}\\
        \arout((q_\bos,q_a,q_b))(b)&=\top\iff \text{$q_a=\top$}\\
        \arout((q_\bos,q_a,q_b))(\eos)&=\top\iff \text{$q_\bos=\top$ or $q_b=\top$}.
    \end{align*}
    This defines $(ab)^*$.
    
    But a formula in $\TL[\tlY]$ cannot distinguish between strings that differ beyond their last $k$ symbols (for some constant $k$ depending on the formula), and for any $k$, we have $ab(ab)^{\lceil k/2\rceil} \in (ab)^*$ but $ba(ab)^{\lceil k/2\rceil} \not\in (ab)^*$. A formula in $\TL[\tlH]$ is equivalent to one in $\TL[\tlP]$, which can only define a stutter-invariant language, that is, a language $L$ such that for all $\stru, \sym, \strv$, we have $\stru \sym \strv \in L \iff \stru \sym\sym \strv \in L$ \citep{peled1997stutter}. And $(ab)^*$ is not stutter-invariant, because $ab \in (ab)^*$ but $aab \not\in (ab)^*$.\looseness=-1
\end{proof}

Consequently, $(ab)^*$ is definable by leftmost-hard $\uhat$s and fixed-precision $\smat$s as autoregressors, but not as classifiers.
However, the expressiveness added by autoregression remains limited, as $(aab)^*$ is not definable.

\begin{restatable}{proposition}{aab}\label{prop:aab_not_tl}
    The language $(aab)^*$ is not definable by any $\TL[\tlH]$ classifier or autoregressor.
\end{restatable}
\begin{proof}
    See \cref{sec:aab_inexpressivity}. We show that to distinguish $aab$ from $aaab$, we need at least two nested $\tlY$ operators. But the conversion from an autoregressor to a classifier (\cref{thm:ltl_equiv_altl}\labelcref{thm:altl_to_ltl}) adds only one $\tlY$, so $(aab)^*$ is not definable by any $\TL[\tlH]$ autoregressor.
\end{proof}
Consequently, $(aab)^*$ is not definable by any leftmost-hard $\uhat$ or fixed-precision $\smat$, either as autoregressors or classifiers.

\subsubsection{Temporal Logic with Counting}
\label{sec:crasp}

Other formalisms besides the ones discussed above have been proposed for comparison with transformers. 
\Citet{yang-etal-2025-knee} prove that \smats,  with fixed precision outside attention and arbitrary precision inside attention, are equivalent to a temporal logic with counting operators, $\TL[\countl]$. 
They considered the family of languages

{
\setlength{\abovedisplayskip}{0pt}
\setlength{\belowdisplayskip}{0pt}
\setlength{\abovedisplayshortskip}{0pt}
\setlength{\belowdisplayshortskip}{0pt}
\begin{minipage}{0.45\textwidth}
    \begin{align}
        L_1 &= a^*
    \end{align}
\end{minipage}
\hfill
\begin{minipage}{0.45\textwidth}
    \begin{align}
        L_{k+1} &= \begin{cases}
            L_k b^* & \text{$k$ even} \\
            L_k a^* & \text{$k$ odd}
        \end{cases}
    \end{align}
\end{minipage}
}

and showed that, as Boolean classifiers, transformers with depth $k$ can recognize $L_k$ (and not $L_{k+1}$). But their experiments were on the symbol-prediction task (\cref{sec:related}), closely related to Boolean autoregression. They showed both theoretically and experimentally that \smats with depth $k$ can solve the symbol-prediction task for not only $L_k$, but $L_{k+2}$ (and not $L_{k+3}$).
In the present framework, this discrepancy can be readily explained. Like $\TL[\tlH]$, the logic $\TL[\countl]$ lacks a $\tlY$ operator or an equivalent. So it is more expressive as an autoregressor than as a classifier.

\section{Related Work}
\label{sec:related}

Theoretical study of transformers as language models has not gone totally neglected.
\Citet{hahn-2020-theoretical} compared a \smat{} language model with a probabilistic finite automaton for parity (strings that have an odd number of $1$'s).
\Citet{yao-etal-2021-self}, following previous work on RNNs, considered a transformer language model to $\epsilon$-generate a language if it assigns probability at least $\epsilon$ to each symbol in every string in the language (and no strings not in the language).
They also discussed how to convert a construction for a bounded Dyck language (strings of matching parentheses up to a certain depth) from an $\epsilon$-generator to a language recognizer. \Citet{svete-cotterell-2024-transformers} showed that average-hard attention transformer language models can exactly express all $n$-gram language models. 
These studies were specialized to particular languages, or used specialized ways of comparing distributions that do not generalize in an obvious way.

Experimentally, \citet{bhattamishra-etal-2020-ability} proved theoretical results on transformers as language recognizers but 
carried out experiments on transformer language models for the character prediction task: predict, at each position, the set of next possible symbols, that is, Boolean autoregression.
This experimental setup was previously used in studies of RNNs, and has been adopted in other studies of transformers \citep{huang2025a,yang-etal-2025-knee}, which we discussed in \cref{sec:crasp}.
The idea that the sequence of output vectors of a transformer and the states of a finite automaton can be connected via the notion of a \strtostate{} is not new; previous results on using transformers to simulate (weighted) finite automata made a similar connection \citep{Liu-2022-shortcuts,pmlr-v238-rizvi-martel24a}.\looseness=-1

\section{Discussion} We have observed settings where classifiers coincide with autoregressors, and settings where they do not. Where the two do coincide (e.g.,~Boolean-weighted \LTL{} and \uhats), we can now transfer results on definability from the more well-understood world of classifiers to autoregressors. For example,~PARITY is not expressible by \uhat{} classifiers \citep{hahn-2020-theoretical}, and therefore not by \uhat{} autoregressors. Where classifiers and autoregressors do not coincide (e.g.,~with real weights), we do not yet have good techniques for showing inexpressibility by autoregressors, which is left for future work.\looseness=-1

One direction for future work is to extend our results to more realistic classes of transformers by considering log-precision softmax attention or positional encodings. 
The former would require stronger characterizations of log-precision softmax attention transformers. 
The latter could utilize connections between positional encodings and numerical predicates, as discussed by \citet{yang-etal-2024-masked,barcelo-etal-24}.
Another future avenue is studying autoregressors enriched with chain of thought, that is, allowing the autoregressor to run for a number of intermediate steps before producing an answer. 
There are numerous results relating transformers with chain of thought and different classes of computational problems \citep{attention-turing,MS24,BPG20,enable-cot,nowak-etal-2024-representational,constant-cot,hou2025universal}, but only with Boolean weights. 
What can be said about real-weighted autoregressors with chain of thought, and the probability distributions they compute? For example, can they compute those of probabilistic Turing machines solving problems in the BPP or even PP complexity classes?

Lastly, by clarifying the relationship between classifiers and autoregressors, our results provide a principled way to interpet and build upon the various experiments that test the expressive capabilities of transformers \citep[][\textit{inter alia}]{weiss-etal-2018-practical,bhattamishra-etal-2020-ability,van-der-poel-2024-mlregtest,deletang2023neural,someya-etal-2024-targeted,borenstein-etal-2024-languages,ICLR2025_7256b2c0}.
For example, we are able to explain in \cref{sec:crasp} why transformers as classifiers and autoregressors exhibit different expressive capabilities in \citet{yang-etal-2025-knee}'s experiments.
More generally, we have identified several theoretical separations between the expressive capabilities of classifiers and autoregressors in both the Boolean and real-weighted cases, which future work could probe experimentally.

\ificlrfinal

\section*{Acknowledgements}

This paper developed from the findings of the working group on probability at Dagstuhl Seminar 25282, ``Theory of Neural Language Models.'' We are grateful to the Leibniz Center for Informatics for their support.
We also thank Gavin Dooley for his feedback and a correction.
This material is based in part upon work supported by the US National Science Foundation under Grant No.~2502292 and the European Research Council under Grant No.~101089343.
Andy Yang is supported by the US National Science Foundation Graduate Research Fellowship Program under Grant No.~2236418, and
Anej Svete is supported
by the ETH AI Center Doctoral Fellowship.
\fi

\bibliography{references}

\appendix

\include{appendix}

\end{document}

%% file: rycolab-macros.tex
\usepackage{colortbl}
\usepackage[dvipsnames]{xcolor}

\definecolor{ETHBlue}{RGB}{33,92,175}	%
\definecolor{ETHGreen}{RGB}{98,115,19}		%
\definecolor{ETHPurpleDark}{RGB}{140,10,89}	%
\definecolor{ETHPurple}{RGB}{163,7,116}	%
\definecolor{ETHGray}{RGB}{111,111,111}	%
\definecolor{ETHRed}{RGB}{183,53,45}	%
\definecolor{ETHPetrol}{RGB}{0,120,148}	%
\definecolor{ETHBronze}{RGB}{142,103,19}	%

\colorlet{ETHdarkblue}{ETHBlue!80!black}
\colorlet{ETHdarkgreen}{ETHGreen!80!black}
\colorlet{ETHpink}{ETHPurple}
\colorlet{ETHgray}{ETHGray}
\colorlet{ETHred}{ETHRed}
\colorlet{ETHgreenblue}{ETHPetrol}
\colorlet{ETHbrown}{ETHBronze}

\definecolor{TextBlack}{RGB}{51,51,51}
\definecolor{BackgroundWhite}{RGB}{255,255,255}

\definecolor{AccentBlue}{RGB}{0,122,204}
\definecolor{LightBlue}{RGB}{173,216,230}
\definecolor{DarkBlue}{RGB}{0,51,102}

\definecolor{AccentGreen}{RGB}{70,160,73}
\definecolor{LightGreen}{RGB}{144,238,144}
\definecolor{DarkGreen}{RGB}{0,100,0}

\definecolor{AccentRed}{RGB}{255,0,0}
\definecolor{LightRed}{RGB}{255,99,71}
\definecolor{DarkRed}{RGB}{139,0,0}

\definecolor{AccentOrange}{RGB}{255,165,0}
\definecolor{LightOrange}{RGB}{255,204,153}
\definecolor{DarkOrange}{RGB}{255,140,0}

\definecolor{NeutralLightGray}{RGB}{204,204,204}
\definecolor{NeutralMediumGray}{RGB}{102,102,102}

\definecolor{NoteYellow}{RGB}{255,255,0}

\definecolor{DiversePurple}{RGB}{128,0,128}
\definecolor{DiverseTeal}{RGB}{0,128,128}
\definecolor{DiverseOlive}{RGB}{128,128,0}
\definecolor{DiverseCyan}{RGB}{0,128,192}
\definecolor{DiverseMagenta}{RGB}{192,0,128}

\colorlet{MacroColor}{ETHPetrol}
\colorlet{MACROCOLOR}{MacroColor}

\newcommand{\mymacro}[1]{{#1}}

\newcommand{\defn}[1]{\textbf{#1}}

\newcommand{\paroutline}[3][false]{%
    \ifnum\pdfstrcmp{#1}{true}=0
        #3%
    \else
        [\textit{\textcolor{DiverseMagenta}{#2}}] \textcolor{AccentBlue}{#3}%
    \fi
}

\newcommand{\uhat}{\mymacro{\textsf{UHAT}}\xspace}
\newcommand{\uhats}{\mymacro{\textsf{UHATs}}\xspace}

\newcommand{\smat}{\mymacro{\textsf{SMAT}}\xspace}
\newcommand{\smats}{\mymacro{\textsf{SMATs}}\xspace}

\newcommand{\ind}[1]{\mathbb{I} \left\{ #1 \right\}}

\newcommand{\R}{{\mymacro{ \mathbb{R}}}}

\newcommand{\B}{{\mymacro{ \mathbb{B}}}}

\newcommand{\alphabet}{{\mymacro{ \Sigma}}}

\newcommand{\kleene}[1]{{\mymacro{#1^*}}}

\newcommand{\str}{{\mymacro{\boldsymbol{w}}}}

\newcommand{\stru}{{\mymacro{\boldsymbol{u}}}}
\newcommand{\strv}{{\mymacro{\boldsymbol{v}}}}

\newcommand{\bos}{{\mymacro{\textsc{bos}}}}
\newcommand{\eos}{{\mymacro{\textsc{eos}}}}

\newcommand{\semiring}{{\mymacro{ \mathbb{K}}}}

\newcommand{\zero}{{\mymacro{\mathbf{0}}}}
\newcommand{\one}{{\mymacro{\mathbf{1}}}}

\newcommand{\automaton}{{\mymacro{M}}}

\newcommand{\stateq}{{\mymacro{ q}}}

\newcommand{\states}{{\mymacro{ Q}}}

\newcommand{\trans}{{\mymacro{ \delta}}}

\newcommand{\weight}{{\mymacro{k}}}

\newcommand{\final}{{\mymacro{ F}}}

\newcommand{\finalf}{{\mymacro{ F}}}

\newcommand{\edge}[4]{{\mymacro{#1 \xrightarrow{#2 / #3} #4}}}

\newcommand{\negterm}[1]{{\mymacro{ {\raise.17ex\hbox{$\scriptstyle\sim$}} #1}}}

\newcommand{\initstate}{{\mymacro{\iota}}}

\newcommand{\ignore}[1]{}
\newcommand{\expandLater}[1]{}

\def\1{\mathbf{1}}

\def\rmH{{{\mymacro{ \mathbf{H}}}}}

\def\rmK{{{\mymacro{ \mathbf{K}}}}}

\def\rmQ{{{\mymacro{ \mathbf{Q}}}}}

\def\rmV{{{\mymacro{ \mathbf{V}}}}}

\newcommand{\softmax}{{\mymacro{ \mathrm{softmax}}}}

\newcommand{\bb}[1][]{\ifthenelse{\isempty{#1}}{\mymacro{\mathbf{b}}}{\mymacro{\mathbf{b}^{\text{#1}}}}}
\newcommand{\ff}[1][]{\ifthenelse{\isempty{#1}}{\mymacro{f}}{\mymacro{f_{\text{#1}}}}}

\newcommand{\W}[1][]{\ifthenelse{\isempty{#1}}{\mymacro{\mathbf{W}}}{\mymacro{\mathbf{W}^{\text{#1}}}}}

\newcommand{\stacktop}[1][]{\ifthenelse{\isempty{#1}}{\mymacro{\gamma^{\text{top}}}}{\mymacro{\gamma^{\text{top}}_{#1}}}}

\newcommand{\sym}{\mymacro{{\sigma}}}
\newcommand{\syma}{\mymacro{{a}}}

\makeatletter
\newcommand*{\bigcdot}{}%
\DeclareRobustCommand*{\bigcdot}{%
  \mathbin{\mathpalette\bigcdot@{}}%
}
\newcommand*{\bigcdot@scalefactor}{.5}
\newcommand*{\bigcdot@widthfactor}{1.15}
\newcommand*{\bigcdot@}[2]{%
  \sbox0{$#1\vcenter{}$}%
  \sbox2{$#1\cdot\m@th$}%
  \hbox to \bigcdot@widthfactor\wd2{%
    \hfil
    \raise\ht0\hbox{%
      \scalebox{\bigcdot@scalefactor}{%
        \lower\ht0\hbox{$#1\bullet\m@th$}%
      }%
    }%
    \hfil
  }%
}
\makeatother

\usepackage{stackengine}

\newcommand{\true}{\mymacro{\top}}

\newcommand{\false}{\mymacro{\bot}}

\NewDocumentCommand{\transformer}{O{} o o o }{
  \IfBlankTF{#1}{
    \IfNoValueTF{#2}{
      \IfNoValueTF{#4}{\rmH}{\rmH(#4)}
    }{
      \IfNoValueTF{#4}{\mymacro{\rmH_{#2,#3}}}{\mymacro{\rmH(#4)_{#2,#3}}}
    }
  }{
    \IfNoValueTF{#2}{
      \IfNoValueTF{#4}{\mymacro{\rmH^{(#1)}}}{\mymacro{\rmH^{(#1)}(#4)}}
    }{
      \IfNoValueTF{#4}{\mymacro{\rmH^{(#1)}_{#2,#3}}}{\mymacro{\rmH^{(#1)}(#4)_{#2,#3}}}
    }
  }
}
\NewDocumentCommand{\attention}{o}{\IfNoValueTF{#1}{\mymacro{\mathbf{A}}}{\mymacro{\mathbf{A}^{(#1)}}}}
\NewDocumentCommand{\ffn}{o}{\IfNoValueTF{#1}{\mymacro{\mathbf{F}}}{\mymacro{\mathbf{F}^{(#1)}}}}
\NewDocumentCommand{\fo}{o}{\IfNoValueTF{#1}{\mymacro{\textbf{FO}[\mathord<]}}{\mymacro{\textbf{FO}^{#1}[\mathord<]}}}
\NewDocumentCommand{\pfo}{o}{\IfNoValueTF{#1}{\mymacro{\textnormal{\textbf{PFO}}^2[\mathord<]}}{\mymacro{\textnormal{\textbf{PFO}}^2[\mathord<,#1]}}}
\NewDocumentCommand{\ffo}{o}{\IfNoValueTF{#1}{\mymacro{\textnormal{\textbf{FFO}}^2[\mathord<]}}{\mymacro{\textnormal{\textbf{FFO}}^2[\mathord<,#1]}}}

\newcommand{\TL}{\mymacro{\mathsf{TL}}}

\NewDocumentCommand{\query}{o o }{\IfNoValueTF{#1}{\rmQ}{\rmQ_{#1,#2}}}
\NewDocumentCommand{\key}{o o }{\IfNoValueTF{#1}{\rmK}{\rmK_{#1,#2}}}
\NewDocumentCommand{\val}{o o }{\IfNoValueTF{#1}{\rmV}{\rmV_{#1,#2}}}

\newcommand{\fpseries}{\mymacro{S}}

\newcommand{\bnfto}{\mathrel{::=}}
\newcommand{\bnfor}{\mathrel{\makebox[\widthof{$\bnfto$}][r]{$\mid$}}}

\newcommand{\LTL}{\mymacro{\ensuremath{\mathsf{LTL}}}}

\newcommand{\statey}{\mymacro{q}}
\newcommand{\statevec}{\mymacro{\mathbf{h}}}
\newcommand{\statescl}{\mymacro{h}}

\newcommand{\realsemiring}{\mymacro{\overline{\mathbb{R}}_{\ge0}}}

\newcommand{\strtostate}{\mymacro{{state encoder}}}
\newcommand{\strtostates}{\mymacro{{state encoders}}}

\newcommand{\Strtostatestitle}{\mymacro{State Encoders}}

\newcommand{\iffby}[1]{\mathrel{\overset{\smash{\text{\labelcref{#1}}}}{\iff}}}
\newcommand{\iffbyih}{\mathrel{\overset{\smash{\text{ind.~hyp.}}}{\iff}}}

\newcommand{\PSPACE}{\mathsf{PSPACE}}
\newcommand{\PTIME}{\mathsf{P}}
\newcommand{\NPTIME}{\mathsf{NP}}

\newcommand{\autoregressor}{\mymacro{A}}
\newcommand{\arout}{\mymacro{r}}

\newcommand{\tlS}{\mathbin{\mymacro{\mathbf{S}}}}
\newcommand{\tlY}{\mathord{\mymacro{\mathbf{Y}}}}
\newcommand{\tlH}{\mathord{\mymacro{\mathbf{H}}}}
\newcommand{\tlP}{\mathord{\mymacro{\mathbf{P}}}}

%% file: appendix.tex
\section{Equivalence of \Strtostatestitle{}}\label{sec:same_states_proof}

\stateEquivalence*
\begin{proof}%
First we show the equivalence of state sequences defined by \uhats and $\LTL$, and then equivalence of $\LTL$ and cfDFAs.

The essential observation \citep[Lemma~22]{yang-etal-2024-masked} is that the output at every position of every \uhat layer comes from a finite set $Q \subseteq \mathbb{R}^d$. So we can think of a \uhat{} as a function $\mathcal{T} \colon \alphabet^* \to Q^*$.
For each $\statey \in Q$, we can construct an \LTL{} formula $\phi_\statey$ such that $\mathcal{T}(\str)_{i} = \statey \iff \str,i\models \phi_{\statey}$ \citep[Theorems~2, 4]{yang-etal-2024-masked}.
So there exists a tuple of $\LTL$ formulas $(\phi_{\statey})_{\statey \in Q}$ that defines a \strtostate{} equivalent to $\mathcal{T}$. 
Note that the state outputted by $\mathcal{T}$ on the prepended $\bos$ symbol can be simulated using a $\bos$ formula in the tuple. 

In the other direction, for every tuple of $\LTL$ formulas $(\phi_{1},\phi_{2},\ldots, \phi_{m})$ defining a \strtostate{} $\alphabet^*\to \B^m$, there exists a \uhat $\mathcal{T}\colon \alphabet^*\to (\R^d)^*$ defining an equivalent \strtostate. For each $\phi_k$, we construct a transformer $\mathcal{T}_k$ which outputs $\frac{1}{2}$ if $\str,i\models \phi_k$ and $-\frac{1}{2}$ otherwise \citep[Theorems~1, 3]{yang-etal-2024-masked}.
Then we can parallel-compose all the $\mathcal{T}_k$ into a single $\mathcal{T}$ \citep[Lemma~25]{yang-etal-2024-masked}, and add an additional layer which projects the output dimensions of each $\mathcal{T}_k$ into a single output vector $\R^m$ such that $\mathcal{T}(\str)_{i} =\mathbf{e}_k\iff \str,i\models \phi_{k}$.

The equivalence between $\LTL$ and cfDFAs can be described a little more succinctly. 
Given a DFA $\automaton=(\alphabet, Q,\trans, \iota)$, for each state $q\in Q$ there exists a formula $\phi_q$ such that $\str\models\phi_q\iff \trans(\iota,\str)=q$, due to the expressive equivalence of $\LTL$ and cfDFAs \citep{schutzenberger:1965,mcnaughtonpapert1971,kamp:1968}.
The tuple $(\phi_q)_{q \in Q}$ then defines a \strtostate{} equivalent to $\automaton$.
In the other direction, given a tuple of $\LTL$ formulas $(\phi_1,\ldots,\phi_m)$, for each $k \in [m]$ there is an automaton $\automaton_k$ that recognizes the same language as $\phi_k$. 
Then the Cartesian product of all the $\automaton_k$ defines a \strtostate{} equivalent to $(\phi_1, \ldots, \phi_m)$. 
\end{proof}

\section{Autoregressive Model Proofs}

\subsection{Proof of \Cref{thm:ltl_next}}
\label{sec:ltl_next_proof}

\begin{restatable}{lemma}{ltlNext}\label{thm:ltl_next}
There is a transformation $\nextsym{\sym}$ from formulas of $\TL[\mathcal{O}]$ to formulas of $\TL[\mathcal{O}]$ such that for any formula $\phi$ of $\TL[\mathcal{O}]$ and for all $\str \in \kleene{\alphabet}$,
\begin{equation}
\str \models \nextsym{\sym}(\phi) \iff \str \sym \models \phi. \label{eq:next_sem}
\end{equation}
\end{restatable}
Intuitively, $\nextsym{\sym}$ removes a $\sym$ on the right; in other words, $\nextsym{\sym}(\phi)$ defines the right Brzozowski derivative \citep{brzozowski1964derivatives} of the language defined by $\phi$.

\begin{proof}
We define $\nextsym{\sym}$ recursively:
\begin{subequations}
\begin{align}
\nextsym{\sym}(\sym) &= \true \label{eq:next_a} \\
\nextsym{\sym}(\sym') &= \false \qquad \text{if $\sym' \ne \sym$} \label{eq:next_b} \\
\nextsym{\sym}(\sympred{\bos}) &= \false \label{eq:next_bos} \\
\nextsym{\sym}(\neg \phi) &= \neg \nextsym{\sym}(\phi) \label{eq:next_not} \\
\nextsym{\sym}(\phi_1 \wedge \phi_2) &= \nextsym{\sym}(\phi_1) \wedge \nextsym{\sym}(\phi_2) \label{eq:next_and} \\
\nextsym{\sym}(\tlY \phi) &= \phi \label{eq:next_Y} \\
\nextsym{\sym}(\tlH \phi) &= \tlH \phi  \wedge \nextsym{\sym}(\phi) \label{eq:next_H} \\
\nextsym{\sym}(\phi_1 \tlS \phi_2) &= (\nextsym{\sym}(\phi_1) \wedge (\phi_1 \tlS \phi_2)) \vee \nextsym{\sym}(\phi_2). \label{eq:next_S}
\end{align}
\end{subequations}
Note that $\nextsym{\sym}$ never translates a temporal operator into another temporal operator, so it translates formulas of $\mathsf{TL}[\mathcal{O}]$ into formulas of $\mathsf{TL}[\mathcal{O}]$ for any $\mathcal{O}$.

Next, we prove that $\nextsym{\sym}(\phi)$ satisfies \cref{eq:next_sem} by induction on the structure of $\phi$. 

\begin{subequations}
\paragraph{Base Cases.}
If $\phi = \sym$: 
\begin{align}
\str, i \models \nextsym{\sym}(\sym)  
&\iffby{eq:next_a} \str \models \true \\
&\iffby{eq:models_sym} \str\sym \models \sym.
\end{align}
If $\phi = \sym'$ for $\sym' \ne \sym$: 
\begin{align}
\str \models \nextsym{\sym}(\sym')  
&\iffby{eq:next_b} \str \models \false \\
&\iffby{eq:models_sym} \str\sym \models \sym'.
\end{align}
Similarly, if $\phi = \bos$:
\begin{align}
  \str \models \nextsym{\sym}(\bos) 
  &\iffby{eq:next_bos} \str \models \false \\
  &\iffby{eq:models_bos}\str\sym \models \bos.
\end{align}
\end{subequations}

\begin{subequations}
\paragraph{Inductive Cases.} If $\phi = \neg \phi_1$:
\begin{align}
\str \models \nextsym{\sym}(\lnot \phi_1) &\iffby{eq:next_not} \str \models \lnot \nextsym{\sym}(\phi_1) \\
&\iffby{eq:models_not} \str \not \models  \nextsym{\sym}(\phi_1) \\
&\iffbyih \str\sym \not \models \phi_1 \\
&\iffby{eq:models_not}\str\sym \models \lnot \phi_1.
\end{align}
If $\phi = \phi_1 \land \phi_2$:
\begin{align}
\str \models \nextsym{\sym}(\phi_1 \wedge \phi_2)
&\iffby{eq:next_and} \str \models \nextsym{\sym}(\phi_1) \wedge \nextsym{\sym}(\phi_2) \\
&\iffby{eq:models_and} (\str \models \nextsym{\sym}(\phi_1)) \wedge (\str \models \nextsym{\sym}(\phi_2)) \\
&\iffbyih (\str\sym \models \phi_1) \wedge (\str\sym \models \phi_2) \\
&\iffby{eq:models_and}\str\sym \models \phi_1 \wedge \phi_2.
\end{align}
If $\phi = \tlY \phi_1$:
\begin{align}
   \str \models \nextsym{\sym}(\tlY \phi_1) 
   &\iffby{eq:next_Y} \str \models  \phi_1  \\
   &\iffby{eq:models_Y}\str\sym  \models\tlY \phi_1.
\end{align}
If $\phi = \tlH \phi_1$:
\begin{align}
   \str \models \nextsym{\sym}(\tlH \phi_1)
  &\iffby{eq:next_H} \str  \models\tlH \phi_1  \wedge \nextsym{\sym}(\phi) \\
  &\iffby{eq:models_and} (\str \models\tlH \phi_1) \land (\str \models \nextsym{\sym}(\phi)) \\
  &\iffbyih (\str  \models\tlH \phi_1)  \wedge (\str\sym \models \phi_1) \\
  &\iffby{eq:models_H} \str\sym  \models\tlH \phi_1.
\end{align}
If $\phi = \phi_1 \tlS \phi_2$:
\begin{align}
   \str \models \nextsym{\sym}(\phi_1 \tlS \phi_2) \hspace{-1in} \notag \\
    &\iffby{eq:next_S} \str \models (\nextsym{\sym}(\phi_1) \wedge (\phi_1 \tlS \phi_2)) \vee \nextsym{\sym}(\phi_2) \\
    &\iffby{eq:models_and,eq:models_not} (\str \models \nextsym{\sym}(\phi_1) \wedge (\str \models \phi_1 \tlS \phi_2)) \vee (\str \models \nextsym{\sym}(\phi_2)) \\
    &\iffbyih ((\str\sym \models \phi_1) \wedge (\str \models \phi_1 \tlS \phi_2)) \vee (\str\sym \models \phi_2) \\
    &\iffby{eq:models_S} \str\sym  \models \phi_1 \tlS \phi_2.
\end{align}
\end{subequations}
\end{proof}

\subsection{Proof of \Cref{thm:ltl_prefix}}
\label{sec:ltl_prefix_proof_dfa}

\begin{restatable}{lemma}{ltlRel} \label{thm:ltl_prefix}
    There is a transformation $\pref$ from formulas of $\TL[\mathcal{O}]$ to formulas of $\TL[\mathcal{O}]$ such that for any formula $\phi$ of $\TL[\mathcal{O}]$ and for all $\stru \in \kleene{\alphabet}$,
\begin{equation}
\stru \models \pref(\phi) \iff \text{there exists $\strv \in \kleene{\alphabet}$ such that $\stru\strv \models \phi$}. \label{eq:def_pref}
\end{equation}
\end{restatable}

\newcommand{\tltrans}[1]{\Psi \xrightarrow{\sym} #1}
\newcommand{\cl}{\mathrm{cl}}

\begin{proof}
Given a formula $\phi$ of $\mathsf{TL}[\mathcal{O}]$, let $\cl(\phi)$ be the set of all subformulas of $\phi$ (including $\phi$ itself). Construct a DFA $M_\phi = (2^{\cl(\phi)}, \alphabet, \trans, \initstate, \finalf)$, where
\begin{align}
\initstate &= \{ \chi \in \cl(\phi) \mid \epsilon \models \chi \} \\
\finalf &= \{ \Psi \subseteq \cl(\phi) \mid \phi \in \Psi \} \\
\delta(\Psi, \sym) &= \{ \chi \in \cl(\phi) \mid \tltrans \chi \}
\end{align}
where the relation $\tltrans{\chi}$, which intuitively means that if a string $\str$ satisfies exactly the formulas in $\Psi$, then $\str\sigma$ satisfies $\chi$, is defined as follows:
\begin{subequations}
\begin{align}
\tltrans{\sympred{\sym'}} & \text{ iff $\sym=\sym'$} \label{eq:rel_sym}\\
\tltrans{\chi_1 \land \chi_2} & \text{ iff $\tltrans{\chi_1}$ and $\tltrans{\chi_2}$} \label{eq:rel_and}\\
\tltrans{\lnot \chi} & \text{ iff not $\tltrans \chi$} \label{eq:rel_not}\\
\tltrans{\tlY \chi} & \text{ iff $\chi \in \Psi$} \label{eq:rel_Y}\\
\tltrans{\tlH \chi} & \text{ iff $\tlH \chi \in \Psi$ and $\tltrans \chi$} \label{eq:rel_H}\\
\tltrans{\chi_1 \tlS \chi_2} & \text{ iff ($\chi_1 \tlS \chi_2 \in \Psi$ and $\tltrans{\chi_1}$) or $\tltrans{\chi_2}$} \label{eq:rel_S}.
\end{align}
\end{subequations}
\begin{claim}
For any $\str \in \alphabet^*$, if $\Psi = \{ \chi \in \cl(\phi) \mid \str \models \chi \}$, then $\tltrans{\chi} \iff \str\sym \models \chi$.
\end{claim}
\begin{proof} By induction on the structure of $\chi$. 
Note by definition that $\chi\in\Psi\iff \str\models\chi$.
\begin{align*}
\tltrans{\sympred{\sym'}} &\iffby{eq:rel_sym} \sym = \sym' \\
&\iffby{eq:models_sym} \str\sym \models \sympred{\sym'}. \\
\tltrans{\chi_1 \land \chi_2} 
&\iffby{eq:rel_and} \tltrans{\chi_1}\ \text{and}\ \tltrans{\chi_2} \\ &\iffbyih \str\sym \models \chi_1\ \text{and}\ \str\sym \models \chi_2 \\
&\iffby{eq:models_and}\str\sym \models \chi_1 \land \chi_2. \\
\tltrans{\lnot \chi} 
&\iffby{eq:rel_not} \text{not}\ \tltrans{\chi} \\ &\iffbyih\text{not}\ \str\sym \models \chi \\
&\iffby{eq:models_not} \str\sym \models \lnot \chi. \\
\tltrans{\tlY \chi} &\iffby{eq:rel_Y} \chi \in \Psi \\
&\iff \str \models \chi \\
&\iffby{eq:models_Y} \str\sym \models \tlY \chi. \\
\tltrans{\tlH \chi} &\iffby{eq:rel_H} \tlH \chi \in \Psi\ \text{and}\ \tltrans{\chi} \\
&\iffbyih \str \models \tlH \chi\ \text{and}\ \str\sym \models \chi \\
&\iffby{eq:models_H} \str\sym \models \tlH. \\
\tltrans{\chi_1 \tlS \chi_2} &\iffby{eq:rel_S} (\chi_1 \tlS \chi_2 \in \Psi\ \text{and}\ \tltrans{\chi_1})\ \text{or}\ \tltrans{\chi_2} \\
&\iffbyih (\str \models \chi_1 \tlS \chi_2\ \text{and}\ \str\sym \models \chi_1)\ \text{or}\ \str\sym \models \chi_2 \\
&\iffby{eq:models_S} \str\sym \models \chi_1 \tlS \chi_2. \tag*{\qedhere}
\end{align*}
\end{proof}

\begin{claim} \label{thm:tl_to_dfa_prefix}
For any $\str$, $\delta(\initstate, \str) = \{ \chi \in \cl(\phi) \mid \str \models \chi \}$.
\end{claim}
\begin{proof}
By induction on the length of $\str$.

Base case: $\delta(\initstate, \epsilon) = \initstate = \{ \chi \mid \epsilon \models \chi\}$.

Inductive step: Assume that $\delta(\initstate, \str) = \{ \chi \mid \str \models \chi\} = \Psi$.
Then 
\begin{align*}
\delta(\initstate, \str) &= \delta(\delta(\initstate, \str), \sym) \\
&= \delta(\Psi, \sym) \\
&= \{\chi \mid \tltrans{\chi}\} \\
&= \{\chi \mid \str\sym \models \chi\}. \tag*{\qedhere}
\end{align*}
\end{proof}

\begin{claim}\label{claim:same-language}
$M_\phi$ defines the same language as $\phi$.
\end{claim}
\begin{proof}
$\delta(\initstate, \str) \in \finalf$ if and only if $\phi \in \{\chi \mid \str \models \chi\}$ if and only if $\str \models \phi$.
\end{proof}

Then make every co-accessible state (every state that has a path to an accept state) into an accept state. Call this new DFA $M_\phi'$ with accept states~$\finalf'$.
This DFA recognizes the prefix language of $M_\phi$.
Finally, construct the formula \[\pref(\phi) = \bigvee_{\Psi \in \finalf'} \left(\bigwedge_{\chi \in \Psi} \chi \land \bigwedge_{\chi \in \cl(\phi) \setminus \Psi} \lnot \chi\right).\]
Note that $\pref$ never translates a temporal operator into another temporal operator, so it translates formulas of $\mathsf{TL}[\mathcal{O}]$ into formulas of $\mathsf{TL}[\mathcal{O}]$ for any $\mathcal{O}$.

\begin{claim}
The formula $\pref(\phi)$ defines the same language as $M_{\phi}'$.
\end{claim}

\begin{proof}
Since we only changed non-accept states to accept states, \cref{thm:tl_to_dfa_prefix} still applies to $M_\phi'$ and~$\phi$.
\begin{align*}
 \str \in \mathcal{L}(M'_{\phi}) &\iff \delta(\iota, \str) \in \finalf' \\
 &\iff \{ \chi \in \cl(\phi) \mid \str \models \chi \} \in \final' && \text{\cref{thm:tl_to_dfa_prefix}} \\
 &\iff \text{for some $\Psi \in \finalf'$, $\chi \in \Psi$ iff $\str \models \chi$} \\
 &\iff \text{for some $\Psi \in \finalf'$, $\str \models \bigwedge_{\chi \in \Psi} \chi \land \bigwedge_{\chi \in \cl(\phi) \setminus \Psi} \lnot \chi$} \\
 &\iff \str \models \bigvee_{\Psi \in \finalf'} \left(\bigwedge_{\chi \in \Psi} \chi \land \bigwedge_{\chi \in \cl(\phi) \setminus \Psi} \lnot \chi \right). %
\end{align*}
\newsavebox{\trash}
\sbox{\trash}{\qedhere}
\end{proof}

This completes the proof of \cref{thm:ltl_prefix}.
\end{proof}

\subsection{Relationship Between Classifiers and Autoregressors}\label{sec:ltl_equiv_altl_proof}

\ltlEquivaltl*
\begin{proof}%

\labelcref{thm:ltl_to_altl} A Boolean-weighted $\mathsf{TL}[\mathcal{O}]$ classifier is defined by a tuple of formulas $\Phi=(\phi_1,\ldots,\phi_m)$ inducing a state encoder and an output function $c\colon \B^{m}\to \B$.
We may think of $c$ as a Boolean combination of its arguments, and substitute the $\phi_i$ into it to obtain a single formula $\phi = c(\phi_1, \ldots, \phi_m)$. Then define a new trivial output function $c'(\statescl) = \statescl$, so that the classifier $(\phi, c')$ defines the same language as $(\Phi, c)$.

Define
\begin{align}
\phi_\sym &= \nextsym{\sym}(\pref(\phi)) \qquad \text{for $\sym \in \alphabet$} \\
\phi_\eos &= \phi.
\end{align}
Then the tuple of formulas $\Phi'=(\phi_\sym)_{\sym \in \alphabet \cup \{\eos\}}$ defines a \strtostate{}. We define the autoregressive output function 

\begin{align}
\arout\colon \B^{|\alphabet|+1}&\to \B^{|\alphabet|+1} \\
    \arout(\statevec)(\sym)&=\begin{cases}
    \statescl_\sym & \text{if any entry of $\statevec$ is true} \\
    \top & \text{if all entries of $\statevec$ are false.}
    \end{cases}
    \label{eq:atl_def}
\end{align}
A vector $\statevec$ whose entries are all false is unreachable, so it does not matter what we set $\arout(\statevec)$ to, but it must satisfy $\bigoplus_\sym \arout(\statevec)(\sym) = \one$ (\cref{eq:sym_probability}), that is, the entries of $\arout(\statevec)$ cannot all be false.

The autoregressor $\autoregressor = (\Phi', \arout)$ defines $L$, because for any $\str \in L$ with length $n$, we have
\begin{align}
\strprob{\autoregressor}{\str} &=  \bigotimes_{i=1}^{n} \symprob{\autoregressor}{\str_{<i}}{w_i} \otimes \symprob{\autoregressor}{\str}{\eos}\\
&= \bigwedge_{i=1}^{n}\ind{\str_{<i} \models \nextsym{w_i}(\pref(\phi)} \land \ind{\str \models \phi} \\
&= \bigwedge_{i=1}^{n}\ind{\str_{\le i} \models \pref(\phi)} \land \ind{\str \models \phi} \\
&= \bigwedge_{i=1}^{n}\ind{\text{$\str_{\le i}\strv \models \phi$ for some $\strv$}} \land \ind{\str \models \phi} \\
&= \top.
\end{align}
On the other hand, for any $\str \not\in L$, let $k$ be the greatest integer such that $\str_{<k} \strv \in L$ for some $\strv$. (We know that $k$ exists because $L$ is nonempty by assumption.)
Then we have that $\str_{<k} \not \models \nextsym{w_k}(\pref(\phi))$, but 
$\str_{<k} \models \nextsym{v_1}(\pref(\phi))$.
So $\symprob{\autoregressor}{\str_{<k}}{w_k} = \bot$, and therefore $\strprob{\autoregressor}{\str} 
= \bot$.

It remains to verify that $\autoregressor$ satisfies \cref{eq:suffix_probability}. 
For any $\stru \in \alphabet^*$, we want to show that 
\begin{align}
\bigoplus_{\strv} \sufprob{\autoregressor}{\stru}{\strv} &= \bigoplus_{\strv} \left( \bigotimes_{i=1}^{|\strv|} \symprob{\autoregressor}{\stru\strv_{<i}}{v_i} \otimes \symprob{\autoregressor}{\stru\strv}{\eos}\right) = \one. \label{eq:summation}
\end{align}
In the Boolean semiring, it suffices to show that at least one term of this summation is true.

If there is a $\strv$ such that $\stru\strv \in L$, then the corresponding term of the summation is
\begin{align}
&\left(\bigwedge_{i=1}^{|\stru\strv|} \ind{\stru\strv_{<i} \models \nextsym{v_i}(\pref(\phi)) }\right) \land \ind{\stru\strv \models \phi} \\
&= \left( \bigwedge_{i=1}^{|\stru\strv|} \ind{\stru\strv_{\le i} \models \pref(\phi) } \right) \land \ind{\stru\strv \models \phi}  \\
&= \left( \bigwedge_{i=1}^{|\stru\strv|} \ind{ \text{$\stru\strv_{\le i}\str \models \phi$ for some $\str$} } \right) \land \ind{\stru\strv \models \phi}  \\
&= \true.
\end{align}
If there is no such $\strv$, then for all $\sym$, we have 
$\stru\sym \not\models \pref(\phi)$, so $\stru \not\models \nextsym{\sym}(\pref(\phi))$; moreover, $\stru \not\models \phi$. 
Let $\statevec = \Phi'(\stru)_{|\stru|}$ be the state after reading $\stru$.
All entries of $\statevec$ are false, which makes (for example) $a(\statevec)(\eos)$ true,
so the $\strv=\epsilon$ term of the summation in \cref{eq:summation} is true.

\labelcref{thm:altl_to_ltl} Let $\autoregressor = (\Phi, \arout)$ be a $\TL[\mathcal{O}]$ autoregressor, where $\Phi = (\phi_i)_{i=1}^m$.

Define
\begin{align*}
    \phi_\statevec &= \bigwedge_{i=1}^m (\phi_i \leftrightarrow \statescl_i) && \text{for $\statevec \in \B^m$} \\
    \phi_\sym &= \bigwedge_{\statevec \in \B^m} (\phi_\statevec  \leftrightarrow \arout(\statevec)(\sym)) &&\text{for $\sym$ in $\alphabet\cup\{\eos\}$.}
\end{align*}
Intuitively, $\phi_h$ is true whenever the \strtostate{} is in state $\statevec$, and $\phi_\sym$ is true whenever the \strtostate{} is in a state in which $\arout$ predicts $\sym$. 

Then, define
\[\phi'=\tlH\left(\bigvee_{\sym\in\alphabet\cup\{\eos\}} (\tlY\phi_\sym) \land \sympred{\sym} \right)\]
and let $c(\statescl)=\statescl$, so that the classifier $C = (\phi', c)$ defines the same language as $\autoregressor$.

\end{proof}

\subsection{Proof of \Cref{cor:ltl_tails}}\label{sec:ltl_tails_proof}

\tails*

\begin{proof}
\labelcref{subthm:ltl_tails_HY}
Suppose that $\pref'$ exists. For any formula $\phi$ of $\TL[\tlH,\tlY]$, we can test whether $\phi$ is satisfiable by constructing $\pref'(\phi)$ in polynomial time (by assumption) and then testing whether $\epsilon \models \pref'(\phi)$, which can also be done in polynomial time, as shown by \citet[Thm.~8]{FiondaGreco2016}. 
(They assume formulas in negation normal form, but it is easy to generalize their result to formulas not in negation normal form.)
But satisfiability in $\TL[\tlH,\tlY]$ is $\PSPACE$-complete \citep{DeGiacomoVardi2013,FiondaGreco2016}, so this would imply $\PTIME=\PSPACE$.

\labelcref{subthm:ltl_tails_H} Similarly, satisfiability in $\TL[\tlH]$ is $\NPTIME$-complete \citep{FiondaGreco2016}, so the existence of a polynomial-time $\pref'$ would imply $\PTIME=\NPTIME$.

\labelcref{subthm:ltl_tails_Y} Same as the previous case.

\end{proof}

\section{Nondeterministic Finite Automata}
\label{sec:nfa_proof}

We give a single definition of weighted NFAs instead of factoring them into unweighted NFAs and autoregressive output functions.

\begin{definition}[Weighted Nondeterministic Finite Automaton]
    A \defn{weighted nondeterministic finite automaton} is a tuple $\automaton=(\alphabet,Q, \delta, \initstate, \omega)$, where
    \begin{itemize}
        \item $\alphabet$ is an alphabet
        \item $\states$ is a finite set of \defn{states}
        \item $\trans \colon \states \times \alphabet \times \states \to \semiring$ is a \defn{transition function}
        \item $\initstate \in \states$ is the \defn{initial} state
        \item $\omega\colon Q\to\semiring$ is the \defn{accept function}.
    \end{itemize}

     We extend $\delta$ to $\delta^*\colon Q\times \Sigma^*\times Q\to\semiring$:
    \begin{align*}
        \delta^*(q,\epsilon,q)&=\one\\
        \delta^*(q,\epsilon,q')&=\zero \qquad q \ne q' \\
        \delta^*(q_1, \sym \str,q_2)&= \bigoplus_{q\in Q}\delta(q_1,\sym,q)\otimes\delta^*(q,\str,q_2).
    \end{align*}
    Then $\automaton$ accepts $\str$ with weight $k$ iff
    \[ k = \bigoplus_{q_2\in Q} \delta^*(\initstate,\str,q_2)\otimes\omega(q_2).\]
\end{definition}

\begin{definition}
We say that an NFA with transition function $\trans$ is \defn{counter-free} if there exists some $k$ such that for all states $q_1, q_2$ and all strings~$\str$, we have $\trans^*(q_1,\str^{k},q_2)=\trans^*(q_1,\str^{k+1},q_2)$.
\end{definition}

The equivalence between counter-free NFAs and DFAs is well-known \citep[e.g.,][Chapter 5, Exercise 15]{mcnaughtonpapert1971}, but we spell out the proof here using our definitions.
\begin{proposition}
With Boolean weights, counter-free NFAs and counter-free DFAs are equivalent.  
\end{proposition}
\begin{proof}
First, any Boolean-weighted counter-free NFA $\automaton=(\alphabet,Q, \delta, \initstate, \omega)$ can be converted into an equivalent counter-free DFA $\automaton'=(\alphabet,Q', \delta', \initstate')$ together with a classifier output function $\omega' \colon Q' \to \mathbb{B}$. 
This is done by the standard construction \citep{sipser-2013}, letting $Q'=2^Q$ and defining ${\delta'}^*(\bar{q}, \str) = \{r \mid q \in \bar{q},\delta^*(q, \str, r) = \one\}$, $\initstate'=\{\initstate\}$, and $\omega'(\bar{q})=\bigoplus_{q\in\bar{q}}\omega(q)$. 
Then $\trans'^*(\bar{q}, \str^k) = \{r \mid \exists q \in \bar{q}. \delta^*(q, \str^k, r) = \one\} = \{r \mid \exists q \in \bar{q}. \delta^*(q, \str^{k+1}, r) = \one\} = \trans'^*(\bar{q}, \str^{k+1})$.
Thus, $\automaton'$ is also counter-free. 
The other direction is trivial.
\end{proof}

A weighted NFA is determinizable if every pair of states which are \emph{siblings} (can be reached by the same string) are also \emph{twins} (all cycles by the same string have the same weight) \citep{mohri-1997-finite}.
The automaton in \cref{fig:automata-examples}c is counter-free, but not determinizable, because $\stateq_1$ and $\stateq_2$ are siblings (both reachable by $\syma$) but not twins (the $a$-labeled cycles $\edge{\stateq_1}{\syma}{\frac{1}{2}}{\stateq_1}$ and $\edge{\stateq_2}{\syma}{\frac{3}{4}}{\stateq_2}$ on the two states have different weights).

\section{Inexpressibility of $(aab)^*$}\label{sec:aab_inexpressivity}

\aab*

We actually prove a slightly stronger statement.
Define the \defn{$\tlY$-depth} of a formula $\phi$ to be the number of nested $\tlY$ operators in $\phi$.
Then we will prove that $(aab)^*$ is not definable by any formula of $\TL[\tlH,\tlY]$ with $\tlY$-depth~$1$. Since the conversion from an autoregressor to a classifier (\cref{thm:ltl_equiv_altl}\labelcref{thm:altl_to_ltl}) adds a single $\tlY$, we will conclude that $(aab)^*$ is not definable by any $\TL[\tlH]$ autoregressor.

\newcommand{\noy}{\textnormal{noy}}
\newcommand{\bigram}{\textnormal{Bigram}}

\begin{lemma} \label{thm:noy}
For any language $L$ over $\Sigma$, define $\bigram(L) = \{ (\bos, w_1) \cdot (w_1, w_2) \cdot (w_2, w_3) \cdots (w_{n-1}, w_{n}) \cdot (w_n, \eos) \mid w \in L \}$.
If $\phi$ is a formula of $\TL[\tlH,\tlY]$ with $\tlY$-depth 1, then there is a formula $\noy(\phi)$ of $\TL[\tlH]$ over $(\alphabet\cup\{\bos\})\times (\alphabet\cup\{\eos\})$ such that $\mathcal{L}(\noy(\phi)) \cap \bigram(\alphabet^*) = \bigram(\mathcal{L}(\phi))$.
\end{lemma}
\begin{proof}
Define the transformation $\noy$, which pushes $\tlY$ down to the atomic formulas, then modifies the atomic formulas to operate on bigrams.
\begin{align*}
\noy(\neg \psi) &= \neg \noy(\psi) &
\noy(\tlY (\neg \psi)) &= \neg \noy(\tlY \psi) \\
\noy(\psi_1 \land \psi_2) &= \noy(\psi_1) \land \noy(\psi_2) &
\noy(\tlY (\psi_1 \land \psi_2)) &= \noy(\tlY \psi_1) \land \noy(\tlY \psi_2) \\
\noy(\tlH \psi) &= \tlH (\noy(\psi)) &
\noy(\tlY (\tlH \psi)) &= \tlH(\noy(\tlY \psi)) \\
\noy(\sympred{\sym}) &= \bigvee_{\sym'\in\alphabet\cup\{\bos\}}\sympred{(\sym',\sym)} &
\noy(\tlY \sympred{\sym}) &= \bigvee_{\sym'\in\alphabet\cup\{\eos\}}\sympred{(\sym,\sym')} \\
\noy(\sympred{\bos}) &= \sympred{\bos} &
\noy(\tlY\,\sympred{\bos}) &= \bigvee_{\sym'\in\alphabet\cup\{\eos\}}\sympred{(\bos,\sym')}. \tag*{\qedhere}
\end{align*}
\end{proof}
Then, to prove that $(aab)^*$ is not definable in $\TL[\tlH,\tlY]$ with $\tlY$-depth 1, suppose it is definable by $\phi$. By \cref{thm:noy}, there is a formula $\noy(\phi)$ of $\TL[\tlH]$ such that \[(\bos, a)\cdot(a,a)\cdot(a,b) \cdot (b,\eos) \in \mathcal{L}(\noy(\phi)) \cap \bigram(\alphabet^*) .\]
But $\mathcal{L}(\noy(\phi))$ must be stutter-invariant \citep{peled1997stutter}, so we also have
 \[(\bos, a)\cdot(a,a)\cdot(a,a)\cdot(a,b) \cdot (b,\eos) \in \mathcal{L}(\noy(\phi)) \cap \bigram(\alphabet^*) .\]
 But this is not in $\bigram((aab)^*)$, which is a contradiction.

%% file: main.bbl
\begin{thebibliography}{45}
\providecommand{\natexlab}[1]{#1}
\providecommand{\url}[1]{\texttt{#1}}
\expandafter\ifx\csname urlstyle\endcsname\relax
  \providecommand{\doi}[1]{doi: #1}\else
  \providecommand{\doi}{doi: \begingroup \urlstyle{rm}\Url}\fi

\bibitem[Barcelo et~al.(2024)Barcelo, Kozachinskiy, Lin, and Podolskii]{barcelo-etal-24}
Pablo Barcelo, Alexander Kozachinskiy, Anthony~W. Lin, and Vladimir Podolskii.
\newblock Logical languages accepted by transformer encoders with hard attention.
\newblock In \emph{Proceedings of the 12th International Conference on Representation Learning (ICLR)}, volume 2024, pages 22077--22087, 2024.
\newblock URL \url{https://proceedings.iclr.cc/paper_files/paper/2024/file/5f0fdc1acd47431f7f3bb8ee85598cef-Paper-Conference.pdf}.

\bibitem[Bhattamishra et~al.(2020{\natexlab{a}})Bhattamishra, Ahuja, and Goyal]{bhattamishra-etal-2020-ability}
Satwik Bhattamishra, Kabir Ahuja, and Navin Goyal.
\newblock On the ability and limitations of {T}ransformers to recognize formal languages.
\newblock In \emph{Proceedings of the 2020 Conference on Empirical Methods in Natural Language Processing (EMNLP)}, pages 7096--7116, 2020{\natexlab{a}}.
\newblock \doi{10.18653/v1/2020.emnlp-main.576}.

\bibitem[Bhattamishra et~al.(2020{\natexlab{b}})Bhattamishra, Patel, and Goyal]{BPG20}
Satwik Bhattamishra, Arkil Patel, and Navin Goyal.
\newblock On the computational power of transformers and its implications in sequence modeling.
\newblock In \emph{Proceedings of the 24th Conference on Computational Natural Language Learning}, pages 455--475, 2020{\natexlab{b}}.
\newblock \doi{10.18653/v1/2020.conll-1.37}.
\newblock URL \url{https://aclanthology.org/2020.conll-1.37/}.

\bibitem[Borenstein et~al.(2024)Borenstein, Svete, Chan, Valvoda, Nowak, Augenstein, Chodroff, and Cotterell]{borenstein-etal-2024-languages}
Nadav Borenstein, Anej Svete, Robin Chan, Josef Valvoda, Franz Nowak, Isabelle Augenstein, Eleanor Chodroff, and Ryan Cotterell.
\newblock What languages are easy to language-model? {A} perspective from learning probabilistic regular languages.
\newblock In \emph{Proceedings of the 62nd Annual Meeting of the Association for Computational Linguistics}, pages 15115--15134, 2024.
\newblock \doi{10.18653/v1/2024.acl-long.807}.

\bibitem[Brzozowski(1964)]{brzozowski1964derivatives}
Janusz~A. Brzozowski.
\newblock Derivatives of regular expressions.
\newblock \emph{Journal of the Association for Computing Machinery}, 11\penalty0 (4):\penalty0 481–494, October 1964.
\newblock \doi{10.1145/321239.321249}.

\bibitem[Butoi et~al.(2025)Butoi, Khalighinejad, Svete, Valvoda, Cotterell, and DuSell]{ICLR2025_7256b2c0}
Alexandra Butoi, Ghazal Khalighinejad, Anej Svete, Josef Valvoda, Ryan Cotterell, and Brian DuSell.
\newblock Training neural networks as recognizers of formal languages.
\newblock In \emph{Proceedings of the 13th International Conference on Learning Representations (ICLR)}, volume 2025, pages 46273--46316, 2025.
\newblock URL \url{https://openreview.net/forum?id=aWLQTbfFgV}.

\bibitem[Del{\'e}tang et~al.(2023)Del{\'e}tang, Ruoss, Grau-Moya, Genewein, Wenliang, Catt, Cundy, Hutter, Legg, Veness, and Ortega]{deletang2023neural}
Gr{\'e}goire Del{\'e}tang, Anian Ruoss, Jordi Grau-Moya, Tim Genewein, Li~Kevin Wenliang, Elliot Catt, Chris Cundy, Marcus Hutter, Shane Legg, Joel Veness, and Pedro~A. Ortega.
\newblock Neural networks and the {C}homsky hierarchy.
\newblock In \emph{Proceedings of the 11th International Conference on Learning Representations (ICLR)}, 2023.
\newblock URL \url{https://openreview.net/forum?id=WbxHAzkeQcn}.

\bibitem[Devlin et~al.(2019)Devlin, Chang, Lee, and Toutanova]{devlin2019bert}
Jacob Devlin, Ming-Wei Chang, Kenton Lee, and Kristina Toutanova.
\newblock {BERT}: Pre-training of deep bidirectional {T}ransformers for language understanding.
\newblock In \emph{Proceedings of the 2019 Conference of the North American Chapter of the Association for Computational Linguistics: Human Language Technologies (NAACL HLT)}, pages 4171--4186, 2019.
\newblock \doi{10.18653/v1/N19-1423}.

\bibitem[Droste and Gastin(2008)]{droste-gastin-2008}
Manfred Droste and Paul Gastin.
\newblock On aperiodic and star-free formal power series in partially commuting variables.
\newblock \emph{Theory of Computing Systems}, 42\penalty0 (4):\penalty0 608--631, 2008.
\newblock \doi{10.1007/s00224-007-9064-z}.

\bibitem[Droste and Gastin(2019)]{droste-gastin-2019-aperiodic}
Manfred Droste and Paul Gastin.
\newblock Aperiodic weighted automata and weighted first-order logic.
\newblock In \emph{Proceedings of the 44th International Symposium on Mathematical Foundations of Computer Science}, 2019.
\newblock \doi{10.4230/LIPICS.MFCS.2019.76}.

\bibitem[Droste and Kuich(2009)]{droste-kuich-2009}
Manfred Droste and Werner Kuich.
\newblock Semirings and formal power series.
\newblock In Manfred Droste, Werner Kuich, and Heiko Vogler, editors, \emph{Handbook of Weighted Automata}, pages 3--28. Springer Berlin Heidelberg, Berlin, Heidelberg, 2009.
\newblock ISBN 978-3-642-01492-5.
\newblock \doi{10.1007/978-3-642-01492-5_1}.

\bibitem[Droste and Kuske(2021)]{DK21}
Manfred Droste and Dietrich Kuske.
\newblock Weighted automata.
\newblock In Jean{-}{\'{E}}ric Pin, editor, \emph{Handbook of Automata Theory}, pages 113--150. European Mathematical Society Publishing House, Z{\"{u}}rich, Switzerland, 2021.
\newblock \doi{10.4171/AUTOMATA-1/4}.

\bibitem[Fionda and Greco(2016)]{FiondaGreco2016}
Valeria Fionda and Gianluigi Greco.
\newblock The complexity of {LTL} on finite traces: Hard and easy fragments.
\newblock In \emph{Proceedings of the AAAI Conference on Artificial Intelligence}, volume~30, pages 971--977, 2016.
\newblock \doi{10.1609/aaai.v30i1.10104}.

\bibitem[Giacomo and Vardi(2013)]{DeGiacomoVardi2013}
Giuseppe~De Giacomo and Moshe~Y. Vardi.
\newblock Linear temporal logic and linear dynamic logic on finite traces.
\newblock In \emph{Proceedings of the 23rd International Joint Conference on Artificial Intelligence (IJCAI)}, pages 854--860, 2013.
\newblock URL \url{https://www.ijcai.org/Proceedings/13/Papers/132.pdf}.

\bibitem[Goranko and Rumberg(2025)]{sep-temporal-logic}
Valentin Goranko and Antje Rumberg.
\newblock {Temporal Logic}.
\newblock In Edward~N. Zalta and Uri Nodelman, editors, \emph{The {S}tanford Encyclopedia of Philosophy}. Metaphysics Research Lab, Stanford University, {S}ummer 2025 edition, 2025.
\newblock URL \url{https://plato.stanford.edu/archives/sum2025/entries/logic-temporal/}.

\bibitem[Hahn(2020)]{hahn-2020-theoretical}
Michael Hahn.
\newblock Theoretical limitations of self-attention in neural sequence models.
\newblock \emph{Transactions of the Association for Computational Linguistics}, 8:\penalty0 156--171, 2020.
\newblock \doi{10.1162/tacl_a_00306}.

\bibitem[Hou et~al.(2025)Hou, Malach, Jelassi, Brandfonbrener, and Kakade]{hou2025universal}
Kaiying Hou, Eran Malach, Samy Jelassi, David Brandfonbrener, and Sham~M. Kakade.
\newblock Universal length generalization with turing programs.
\newblock In \emph{Proceedings of the 42nd International Conference on Machine Learning (ICML)}, 2025.
\newblock URL \url{https://openreview.net/forum?id=RNSd6G3lcD}.

\bibitem[Huang et~al.(2025)Huang, Yang, Bhattamishra, Sarrof, Krebs, Zhou, Nakkiran, and Hahn]{huang2025a}
Xinting Huang, Andy Yang, Satwik Bhattamishra, Yash Sarrof, Andreas Krebs, Hattie Zhou, Preetum Nakkiran, and Michael Hahn.
\newblock A formal framework for understanding length generalization in transformers.
\newblock In \emph{Proceedings of the 13th International Conference on Learning Representations (ICLR)}, 2025.
\newblock URL \url{https://openreview.net/forum?id=U49N5V51rU}.

\bibitem[Jerad et~al.(2025)Jerad, Svete, Li, and Cotterell]{jerad-etal-2025-unique}
Selim Jerad, Anej Svete, Jiaoda Li, and Ryan Cotterell.
\newblock Unique hard attention: A tale of two sides.
\newblock In \emph{Proceedings of the 63rd Annual Meeting of the Association for Computational Linguistics (ACL)}, pages 977--996, 2025.
\newblock \doi{10.18653/v1/2025.acl-short.76}.

\bibitem[Kamp(1968)]{kamp:1968}
Johan Anthony~Willem Kamp.
\newblock \emph{Tense Logic and the Theory of Linear Order}.
\newblock PhD thesis, University of California, Los Angeles, 1968.
\newblock URL \url{https://www.proquest.com/docview/302320357}.

\bibitem[Li and Cotterell(2025)]{li-cotterell-2025-characterizing}
Jiaoda Li and Ryan Cotterell.
\newblock Characterizing the expressivity of transformer language models.
\newblock In \emph{Advances in Neural Information Processing Systems (NeurIPS) 38}, 2025.
\newblock URL \url{https://arxiv.org/abs/2505.23623}.
\newblock To appear.

\bibitem[Li and Wang(2025)]{constant-cot}
Qian Li and Yuyi Wang.
\newblock Constant bit-size transformers are {T}uring complete.
\newblock In \emph{Advances in Neural Information Processing Systems (NeurIPS) 38}, 2025.
\newblock URL \url{https://arxiv.org/abs/2506.12027}.
\newblock To appear.

\bibitem[Li et~al.(2024)Li, Liu, Zhou, and Ma]{enable-cot}
Zhiyuan Li, Hong Liu, Denny Zhou, and Tengyu Ma.
\newblock Chain of thought empowers transformers to solve inherently serial problems.
\newblock In \emph{Proceedings of the 12th International Conference on Learning Representations (ICLR)}, 2024.
\newblock URL \url{https://openreview.net/forum?id=3EWTEy9MTM}.

\bibitem[Liu et~al.(2023)Liu, Ash, Goel, Krishnamurthy, and Zhang]{Liu-2022-shortcuts}
Bingbin Liu, Jordan~T. Ash, Surbhi Goel, Akshay Krishnamurthy, and Cyril Zhang.
\newblock Transformers learn shortcuts to automata.
\newblock In \emph{Proceedings of the 11th International Conference on Learning Representations (ICLR)}, 2023.
\newblock URL \url{https://openreview.net/forum?id=De4FYqjFueZ}.

\bibitem[Mandrali and Rahonis(2013)]{mandrali2013characterizations}
Eleni Mandrali and George Rahonis.
\newblock Characterizations of weighted first-order logics over semirings.
\newblock In \emph{Algebraic Informatics: 5th International Conference (CAI)}, pages 247--259, 2013.
\newblock \doi{10.1007/978-3-642-40663-8_23}.

\bibitem[Mandrali and Rahonis(2015)]{mandrali2015weighted}
Eleni Mandrali and George Rahonis.
\newblock Weighted first-order logics over semirings.
\newblock \emph{Acta Cybernetica}, 22\penalty0 (2):\penalty0 435--483, 2015.
\newblock \doi{10.14232/actacyb.22.2.2015.13}.

\bibitem[McNaughton and Papert(1971)]{mcnaughtonpapert1971}
Robert McNaughton and Seymour Papert.
\newblock \emph{Counter-Free Automata}.
\newblock Number~65 in M.I.T. Press Research Monographs. M.I.T. Press, 1971.
\newblock URL \url{https://archive.org/details/CounterFre_00_McNa}.

\bibitem[Merrill and Sabharwal(2024)]{MS24}
William Merrill and Ashish Sabharwal.
\newblock The expressive power of transformers with chain of thought.
\newblock In \emph{Proceedings of the 12th International Conference on Learning Representations (ICLR)}, 2024.
\newblock URL \url{https://openreview.net/forum?id=NjNGlPh8Wh}.

\bibitem[Mohri(1997)]{mohri-1997-finite}
Mehryar Mohri.
\newblock Finite-state transducers in language and speech processing.
\newblock \emph{Computational Linguistics}, 23\penalty0 (2):\penalty0 269--311, 1997.
\newblock URL \url{https://aclanthology.org/J97-2003/}.

\bibitem[Nowak et~al.(2024)Nowak, Svete, Butoi, and Cotterell]{nowak-etal-2024-representational}
Franz Nowak, Anej Svete, Alexandra Butoi, and Ryan Cotterell.
\newblock On the representational capacity of neural language models with chain-of-thought reasoning.
\newblock In \emph{Proceedings of the 62nd Annual Meeting of the Association for Computational Linguistics}, pages 12510--12548, 2024.
\newblock \doi{10.18653/v1/2024.acl-long.676}.

\bibitem[Peled and Wilke(1997)]{peled1997stutter}
Doron Peled and Thomas Wilke.
\newblock Stutter-invariant temporal properties are expressible without the next-time operator.
\newblock \emph{Information Processing Letters}, 63\penalty0 (5):\penalty0 243--246, 1997.
\newblock \doi{10.1016/S0020-0190(97)00133-6}.

\bibitem[P{\'{e}}rez et~al.(2021)P{\'{e}}rez, Barcel{\'{o}}, and Marinkovic]{attention-turing}
Jorge P{\'{e}}rez, Pablo Barcel{\'{o}}, and Javier Marinkovic.
\newblock Attention is {T}uring-complete.
\newblock \emph{Journal of Machine Learning Research}, 22:\penalty0 75:1--75:35, 2021.
\newblock URL \url{http://jmlr.org/papers/v22/20-302.html}.

\bibitem[Rizvi-Martel et~al.(2024)Rizvi-Martel, Lizaire, Lacroce, and Rabusseau]{pmlr-v238-rizvi-martel24a}
Michael Rizvi-Martel, Maude Lizaire, Clara Lacroce, and Guillaume Rabusseau.
\newblock Simulating weighted automata over sequences and trees with transformers.
\newblock In \emph{Proceedings of The 27th International Conference on Artificial Intelligence and Statistics (AISTATS)}, pages 2368--2376, 2024.
\newblock URL \url{https://proceedings.mlr.press/v238/rizvi-martel24a.html}.

\bibitem[Sakarovitch(2009)]{Sakarovitch2009}
Jacques Sakarovitch.
\newblock Rational and recognisable power series.
\newblock In Manfred Droste, Werner Kuich, and Heiko Vogler, editors, \emph{Handbook of Weighted Automata}, pages 105--174. Springer, 2009.
\newblock \doi{10.1007/978-3-642-01492-5_4}.

\bibitem[Schützenberger(1965)]{schutzenberger:1965}
M.~P. Schützenberger.
\newblock On finite monoids having only trivial subgroups.
\newblock \emph{Information and Control}, 8\penalty0 (2):\penalty0 190--194, 1965.
\newblock \doi{10.1016/S0019-9958(65)90108-7}.

\bibitem[Sipser(2013)]{sipser-2013}
Michael Sipser.
\newblock \emph{Introduction to the Theory of Computation}.
\newblock Cengage, Boston, MA, third edition, 2013.
\newblock ISBN 113318779X.

\bibitem[Someya et~al.(2024)Someya, Yoshida, and Oseki]{someya-etal-2024-targeted}
Taiga Someya, Ryo Yoshida, and Yohei Oseki.
\newblock Targeted syntactic evaluation on the {C}homsky hierarchy.
\newblock In \emph{Proceedings of the 2024 Joint International Conference on Computational Linguistics, Language Resources and Evaluation (LREC-COLING 2024)}, pages 15595--15605, 2024.
\newblock URL \url{https://aclanthology.org/2024.lrec-main.1356/}.

\bibitem[Strobl et~al.(2024)Strobl, Merrill, Weiss, Chiang, and Angluin]{strobl-etal-2024-survey}
Lena Strobl, William Merrill, Gail Weiss, David Chiang, and Dana Angluin.
\newblock What formal languages can transformers express? {A} survey.
\newblock \emph{Transactions of the Association for Computational Linguistics}, 12:\penalty0 543--561, 2024.
\newblock \doi{10.1162/tacl_a_00663}.

\bibitem[Svete and Cotterell(2024)]{svete-cotterell-2024-transformers}
Anej Svete and Ryan Cotterell.
\newblock Transformers can represent $n$-gram language models.
\newblock In \emph{Proceedings of the 2024 Conference of the North American Chapter of the Association for Computational Linguistics: Human Language Technologies}, pages 6845--6881, 2024.
\newblock \doi{10.18653/v1/2024.naacl-long.381}.

\bibitem[van~der Poel et~al.(2024)van~der Poel, Lambert, Kostyszyn, Gao, Verma, Andersen, Chau, Peterson, Clair, Fodor, Shibata, and Heinz]{van-der-poel-2024-mlregtest}
Sam van~der Poel, Dakotah Lambert, Kalina Kostyszyn, Tiantian Gao, Rahul Verma, Derek Andersen, Joanne Chau, Emily Peterson, Cody~St. Clair, Paul Fodor, Chihiro Shibata, and Jeffrey Heinz.
\newblock {MLR}eg{T}est: A benchmark for the machine learning of regular languages.
\newblock \emph{Journal of Machine Learning Research}, 25\penalty0 (283):\penalty0 1--45, 2024.
\newblock URL \url{https://jmlr.org/papers/v25/23-0518.html}.

\bibitem[Weiss et~al.(2018)Weiss, Goldberg, and Yahav]{weiss-etal-2018-practical}
Gail Weiss, Yoav Goldberg, and Eran Yahav.
\newblock On the practical computational power of finite precision {RNN}s for language recognition.
\newblock In \emph{Proceedings of the 56th Annual Meeting of the Association for Computational Linguistics}, pages 740--745, 2018.
\newblock \doi{10.18653/v1/P18-2117}.

\bibitem[Yang and Chiang(2024)]{yang2024counting}
Andy Yang and David Chiang.
\newblock Counting like transformers: Compiling temporal counting logic into softmax transformers.
\newblock In \emph{Proceedings of the First Conference on Language Modeling (CoLM)}, 2024.
\newblock URL \url{https://openreview.net/forum?id=FmhPg4UJ9K}.

\bibitem[Yang et~al.(2024)Yang, Chiang, and Angluin]{yang-etal-2024-masked}
Andy Yang, David Chiang, and Dana Angluin.
\newblock Masked hard-attention transformers recognize exactly the star-free languages.
\newblock In \emph{Advances in Neural Information Processing Systems (NeurIPS) 37}, pages 10202--10235, 2024.
\newblock URL \url{https://proceedings.neurips.cc/paper_files/paper/2024/hash/13d7f172259b11b230cc5da8768abc5f-Abstract-Conference.html}.

\bibitem[Yang et~al.(2025)Yang, Cadilhac, and Chiang]{yang-etal-2025-knee}
Andy Yang, Micha{\"e}l Cadilhac, and David Chiang.
\newblock Knee-deep in {C-RASP}: A transformer depth hierarchy.
\newblock In \emph{Advances in Neural Information Processing Systems (NeurIPS) 38}, 2025.
\newblock URL \url{https://arxiv.org/abs/2506.16055}.
\newblock To appear.

\bibitem[Yao et~al.(2021)Yao, Peng, Papadimitriou, and Narasimhan]{yao-etal-2021-self}
Shunyu Yao, Binghui Peng, Christos Papadimitriou, and Karthik Narasimhan.
\newblock Self-attention networks can process bounded hierarchical languages.
\newblock In \emph{Proceedings of the 59th Annual Meeting of the Association for Computational Linguistics and the 11th International Joint Conference on Natural Language Processing (ACL-IJCNLP)}, pages 3770--3785, 2021.
\newblock \doi{10.18653/v1/2021.acl-long.292}.

\end{thebibliography}
